\documentclass[pdflatex]{sn-jnl}

\usepackage{graphicx}%
\usepackage{multirow}%
\usepackage{amsmath,amssymb,amsfonts}%
\usepackage{amsthm}%
\usepackage{mathrsfs}%
\usepackage[title]{appendix}%
\usepackage{xcolor}%
\usepackage{textcomp}%
\usepackage{manyfoot}%
\usepackage{booktabs}%
\usepackage{threeparttable}
\usepackage{algorithm}%
\usepackage{algorithmicx}%
\usepackage{algpseudocode}%
\usepackage{listings}%

\usepackage{enumitem}
\usepackage{stmaryrd}
\usepackage{bm} 
\usepackage{float}
\usepackage{graphicx} 
\usepackage{booktabs} 
\usepackage{mathtools}
\usepackage{multirow} 
\usepackage{array} 

\usepackage{bm} 
\usepackage{xcolor} 
\usepackage[normalem]{ulem}
\newcommand{\rev}[1]{#1}
\newcommand{\del}[1]{}
\newcommand{\delmath}[1]{}
\newcommand{\SigmaDiag}{\bm{\Sigma}_{\text{diag}}} 

\theoremstyle{thmstyleone}%
\newtheorem{theorem}{Theorem}
\newtheorem{proposition}{Proposition}

\theoremstyle{thmstyletwo}%
\newtheorem{remark}{Remark}%

\theoremstyle{thmstylethree}%
\newtheorem{definition}{Definition}%
\newtheorem{lemma}{Lemma}
\newtheorem{corollary}{Corollary}

\usepackage[numbers,square,comma]{natbib}
\begin{document}

\title[Localized LoRA-MoE]{
Localized LoRA-MoE: Mixture of Low-Rank\\
Experts With Adaptive Routing}


\author*[1]{\fnm{Babak} \sur{Barazandeh}}\email{bbarazandeh@cribl.io}
\author*[2]{\fnm{Subhabrata} \sur{Majumdar}}\email{smajumdar@iimb.ac.in}
\author[3]{\fnm{Vinay} \sur{Prithyani}}\email{vprithyani3@gatech.edu}
\author[4]{\fnm{George} \sur{Michailidis}}\email{gmichail@ucla.edu}

\affil*[1]{
\orgname{Cribl}}
\affil*[2]{\orgname{Indian Institute of Management Bangalore}}
\affil[3]{\orgname{Georgia Tech}}
\affil[4]{\orgname{University of California Los Angeles}}


\abstract{
\rev{Parameter-efficient fine-tuning (PEFT) methods such as LoRA adapt large models through a single low-rank update. Such updates are susceptible to gradient interference, i.e. when a training stream interleaves multiple operational contexts, conflicting updates to shared adapter weights collapse the adapter into an unspecialized average. While recent methods introduce block-wise isolation to remedy this issue, their topologies are static and cannot adapt to context switching or localized sensor degradation. In this work, we introduce \textbf{Localized LoRA-MoE}, a framework that combines localized spatial blocking with context-conditioned expert routing. We propose and evaluate two architectural variants: \textbf{Block-Wise LoRA-MoE (centralized routing)}, which modulates the entire structural grid through a single context-conditioned gate, and \textbf{Cell-Wise LoRA-MoE (decentralized routing)}, which assigns an independent gate to every coordinate cell of the matrix grid. In experiments spanning three structured transformation benchmarks---a high-dimensional SVD matrix simulation, context-switched transformations over real tabular data, and spatial reconstruction under simulated sensor occlusion---both variants perform better than static baselines under comparable adapter budgets. The variants perform similarly when context shifts the grid uniformly, while cell-wise routing is strictly better when contexts are spatially heterogeneous. These results characterize when routing granularity matters for dynamic structured adaptation. When fine-tuning a language model on six interleaved GLUE tasks, task-conditioned routing at a fixed adapter budget outperforms LoRA and static localized adapters, and shows similar performance to a group of per-task adapters with 6x adapter budgets. At decoder scale, fine-tuning Qwen2.5-1.5B on commonsense reasoning, both routed variants outperform LoRA and the static localized grid across eight held-out benchmarks, with the cell-wise variant best at $+3.9$ points over LoRA at matched active expert capacity.}
}




\maketitle

\section{Introduction}

Large Language Models (LLMs) continue to define the state of the art across natural language processing tasks, benefiting from massive pretraining corpora and increasingly sophisticated architectures. In practical applications, post-training customization of LLMs is often necessary to achieve operationally acceptable performance across specialized domains and structured multi-step agentic workflows~\citep{ding2023peft,barazandeh2026otap,barazandeh2026ontrack}. While full-parameter fine-tuning (FFT) yields strong task adaptation, its computational and storage demands scale linearly with model size, making it impractical for many real-world applications operating under business constraints. This challenge has motivated a growing body of work on \emph{parameter-efficient fine-tuning} (PEFT), which seeks to adapt LLMs by training only a small set of additional parameters \citep{lialin2023scaling}.

Among PEFT methods, Low-Rank Adaptation (LoRA)~\citep{hu2022lora} remains particularly influential: it introduces trainable low-rank updates to frozen weight matrices, offering a simple and effective mechanism for downstream adaptation. Subsequent efforts have refined or extended the LoRA formulation in various ways, including improving parameter sharing~\citep{bishare2025}, boosting-based low-rank updates~\citep{zhang2024less}, residual low-rank pathways~\citep{zhao2025lor2c}, adaptive rank selection~\citep{he2025gora}, structured decompositions~\citep{shi2024lold}, singular-value–aware initialization~\citep{sun2024svfit}, and decoupled update parameterizations~\citep{bini2025delora}. Other works explore geometric generalizations such as hyperbolic fine-tuning~\citep{yang2024hyperbolic}, modular or composable LoRA variants~\citep{huang2023lorahub,liu2024vb}, and vector bank–based adaptation strategies~\citep{kopiczko2023vera,mao2024dora}. Comprehensive surveys such as~\citep{yang2025lowrank} further highlight the breadth of modern LoRA-based techniques.

A major structural deficiency of standard LoRA methods is their susceptibility to global gradient distortion under skewed input statistics. Recent research on modified LoRA architectures~\citep[GraLoRA]{jung2026gralora} has isolated a phenomenon termed \emph{channel dominance}, wherein a small subset of outlier activation channels---frequently observed in modern LLM intermediate representations---disproportionately dictates the update direction. Because standard LoRA constrains parameter updates via a singular global bottleneck, these sparse outlier signals inject destructive gradient entanglement across entirely unrelated input-output pathways, degrading adapter representational capacity at expanded ranks. 

To overcome this bottleneck, a complementary line of research investigates spatial matrix partitioning and architectural sub-blocking. MELoRA~\citep{ren2024melora} demonstrates the immediate computational benefits of localizing updates strictly to diagonal subregions. Localized LoRA~\citep{barazandeh2025localized} generalizes this spatial layout into a formal multi-field partitioning framework, constructing a structured $k \times k$ block grid capable of representing both isolated intra-field correlations and rich cross-field off-diagonal transformations. Concurrently, GraLoRA~\citep{jung2026gralora} utilizes uniform structural matrix slicing to restrict the backward path of outlier channels to localized coordinate sub-blocks, effectively shielding the global parameter landscape from signal corruption. Most recently, blockwise paradigms have extended beyond additive low-rank matrices; BoHA~\citep{yu2025blockwise} leverages a blockwise $W_0$-coupled Hadamard product structure across a spatial grid, optimizing single-task accuracy alongside sequential task-retention metrics.

Parallel to these spatial improvements, PEFT architectures augmented with Mixture-of-Experts (MoE) have emerged. Systems such as MoKA~\citep{yu2025moka}, MoRAL~\citep{yang2024moral}, and LoRAMoE~\citep{dou2023loramoe,dou2025loramoe} highlight how conditional expert routing can insulate modular weights, enhance task specialization, or preserve core world knowledge during instruction tuning. \rev{These systems route over \emph{unstructured} adapters, i.e. a single global gate selects among experts that each update the full weight matrix, with no notion of spatial sub-block structure. Our work addresses this gap by routing the experts over a structured block grid.}

Despite the above improvements to LoRA, an architectural dichotomy persists. Existing structural partitioning techniques (such as Localized LoRA, GraLoRA, and BoHA) remain strictly \emph{static} topologies. They assume a fixed mapping surface across the full operational stream. When confronted with dynamic, context-driven field adjustments—such as operational task switching or active hardware sensor degradation—these static sub-blocks succumb to severe internal parameter conflicts as competing forward states overwrite shared updates. Conversely, traditional MoE-based LoRA variants utilize a singular, global routing configuration that treats the underlying matrix as an unstructured monolith, remaining completely blind to the localized spatial heterogeneity of the layer. Consequently, no existing framework simultaneously exploits fine-grained spatial block isolation \emph{and} input-conditioned dynamic routing.


\paragraph{Our Contributions.}
The above architectural gap motivates our work. We introduce \textbf{Localized LoRA-MoE}, a unified framework that fuses localized spatial blocking with dynamic context routing, resolving optimization deadlocks by replacing rigid and static adapter topologies. Specifically, we map out and evaluate two distinct architectural paradigms for routed spatial reparameterization:

\begin{itemize}[leftmargin=*,nolistsep]
    \item \textbf{Block-Wise LoRA-MoE (Centralized Macro-Routing)}: This architecture deploys multiple spatial block-matrix experts governed by a centralized, global routing network. The macro-router interprets the overarching input context to dynamically swap or modulate entire structural block topologies simultaneously, ensuring high macro-cohesion across the layer surface.
    \item \textbf{Cell-Wise LoRA-MoE (Decentralized Micro-Routing)}: The ultimate architectural evolution. Instead of relying on a single global coordinator, this approach embeds independent, micro-routing units directly within every individual coordinate cell $(i,j)$ of the structured matrix grid. Each localized sub-block cell is granted complete conditional autonomy to select its own low-rank expert pathway based purely on local feature profiles.
\end{itemize}

Through \del{rigorous parameter-parity benchmarking, we explicitly demonstrate} \rev{multi-seed,} \rev{parameter-matched benchmarking, we demonstrate}:
\begin{itemize}[leftmargin=*,nolistsep]
    \item \textbf{Spatially-Isolated Multi-Expert Adaptation}: Both configurations isolate parameter paths into distinct sub-matrices, effectively limiting the global gradient entanglement driven by input activation outliers.
    \item \del{\textbf{Mitigation of Cross-Domain Gradient Warfare}: By introducing conditional activation pathways into partitioned spaces, both methods allow models to absorb severe context-driven domain shifts (e.g., streaming multi-task transitions or hardware sensor degradation) without triggering mutual parameter destruction.}
    \rev{\textbf{Mitigation of Gradient Interference}: By introducing conditional pathways into partitioned spaces, both methods adapt to severe context-driven domain shifts (e.g., streaming multi-task transitions or hardware sensor degradation). The mean pairwise cosine similarity of per-context adapter gradients moves from high negative values under static baselines to close to zero under routing whenever the expert budget covers the context count.}
    \item \del{\textbf{Decentralized Micro-Gating Parity}: Crucially, our empirical results prove that decentralized, cell-level coordinate routing achieves complete statistical parity with an omniscient global router, demonstrating that localized block autonomy can match global capacity without relying on a rigid, single-point routing bottleneck.}
    \rev{\textbf{Routing-Granularity Characterization}: In the experiments, we find that cellwise LoRA-MoE shows similar performance (within standard errors) with blockwise LoRA-MoE when context shifts the grid uniformly, strictly better when contexts are spatially heterogeneous, and substantially more robust under input corruption unseen at training time. We further characterize this advantage through supporting theoretical analysis.}
    \item \del{\textbf{Strict Parameter Efficiency}: Despite the vast increase in architectural flexibility, both the macro-routed and micro-routed multi-expert structures are engineered to operate within strict parameter parity bounds matching standard static configurations.}
    \rev{\textbf{Parameter Efficiency}: Both architecture variants operate at the same expert-adapter budget as static baselines, with routing overhead below $5\%$ of trainable parameters on every benchmark.}
\end{itemize}

Our experiments show that Localized LoRA-MoE---through both Block-Wise Centralized Macro-Routing and Cell-Wise Decentralized Micro-Routing—--significantly improves expressive power under matched parameter budgets\del{, offering a compelling new direction for structured, dynamic, and efficient adaptation of large language models across both global semantic shifts and granular coordinate fields.}\rev{. They cover one synthetic and one real-world tabular dataset, one computer vision dataset, and fine-tuning experiments on pretrained encoder and decoder models.}

\section{Methodology}
\label{sec:methods}

In this section, we first contextualize the architectural evolution by revisiting standard global Low-Rank Adaptation (LoRA) and static multi-field localized partitioning. We then introduce our two proposed dynamic architectures: Block-Wise LoRA-MoE and Cell-Wise LoRA-MoE.

\subsection{Preliminaries}

Let $\mathbf{W}_0 \in \mathbb{R}^{d_{\text{out}} \times d_{\text{in}}}$ denote a pre-trained, frozen weight matrix of a dense neural network layer. For an incoming structural activation input token $\mathbf{x} \in \mathbb{R}^{d_{\text{in}}}$, the nominal feed-forward mapping computes an output representation $\mathbf{h}_0 = \mathbf{W}_0 \mathbf{x}$. 

\paragraph{LoRA Formulation.} Standard Low-Rank Adaptation~\citep{hu2022lora} parametrizes the weight update matrix $\Delta \mathbf{W}$ as an unstructured, global low-rank intrinsic bottleneck:
\begin{equation}
    \mathbf{h} = \mathbf{W}_0 \mathbf{x} + \Delta \mathbf{W} \mathbf{x} = \mathbf{W}_0 \mathbf{x} + \frac{\alpha}{r} \mathbf{B} \mathbf{A}^\top \mathbf{x}
\end{equation}
where $\mathbf{A} \in \mathbb{R}^{d_{\text{in}} \times r}$ and $\mathbf{B} \in \mathbb{R}^{d_{\text{out}} \times r}$ are trainable adapter matrices whose rank $r \ll \min(d_{\text{in}}, d_{\text{out}})$, and $\alpha$ is a constant structural scaling hyperparameter. As isolated by GraLoRA~\citep{jung2026gralora}, when $\mathbf{x}$ features skewed activation profiles or localized channel dominance, this monolithic mapping suffers from gradient entanglement across un-correlated channels, leading to representational collapse at expanded ranks. Such optimization pathologies are well-documented in high-dimensional adversarial settings~\citep{barazandeh2019random}. Moreover, recent spectral and distributional analysis shows that under per-factor norm constraints, the adaptation rank $r$ primarily governs the spectral cost required to cancel leading singular directions of $\mathbf{W}_0$ rather than statistical capacity alone~\citep{barazandeh2026does}, reinforcing why redistributing a fixed rank budget across localized sub-blocks improves expressivity.

\paragraph{Static Structural Block-Partitioning.} To isolate gradient fields, Localized LoRA~\citep{barazandeh2025localized} decomposes the input and output spaces into a collection of sub-blocks. Let the input vector $\mathbf{x}$ be partitioned into $K_x$ contiguous semantic fields, and the updated output vector $\tilde{\mathbf{h}}$ be partitioned into $K_y$ contiguous segments:
\begin{equation}
    \mathbf{x} = \begin{bmatrix} \mathbf{x}_{1}^\top & \mathbf{x}_{2}^\top & \dots & \mathbf{x}_{K_x}^\top \end{bmatrix}^\top, \quad \mathbf{x}_j \in \mathbb{R}^{d_{\text{in}}/K_x}
\end{equation}
\begin{equation}
    \tilde{\mathbf{h}} = \begin{bmatrix} \tilde{\mathbf{h}}_{1}^\top & \tilde{\mathbf{h}}_{2}^\top & \dots & \tilde{\mathbf{h}}_{K_y}^\top \end{bmatrix}^\top, \quad \tilde{\mathbf{h}}_i \in \mathbb{R}^{d_{\text{out}}/K_y}
\end{equation}

By defining localized low-rank parameter components for each coordinate grid cell $(i, j)$ such that $\mathbf{A}_{i,j} \in \mathbb{R}^{(d_{\text{in}}/K_x) \times r_b}$ and $\mathbf{B}_{i,j} \in \mathbb{R}^{(d_{\text{out}}/K_y) \times r_b}$, the static unweighted update to an individual output block $i$ aggregates cross-field mappings:
\begin{equation}
    \tilde{\mathbf{h}}_{i} = \sum_{j=1}^{K_x} \mathbf{B}_{i,j} \mathbf{A}_{i,j}^\top \mathbf{x}_j
\end{equation}
While this layout natively isolates spatial pathways, its mapping surface remains entirely static during inference. When applied to multi-task regimes exhibiting continuous distribution shifts, competing gradient trajectories backpropagating through the shared sub-blocks trigger severe optimization deadlock.

\subsection{Proposal I: Block-Wise LoRA-MoE}

Our first proposed architecture introduces context-conditioned variation at the macro structural scale. We introduce a set of $E$ independent, full-grid localized low-rank experts governed by a centralized routing layer. 

\del{A centralized gating network $\mathbf{W}_g \in \mathbb{R}^{E \times d_{\text{in}}}$ observes the complete, global input activation vector $\mathbf{x}$ to compute a unified macro-routing distribution over the expert pool:}
\rev{Each incoming sample is accompanied by an observed context signal $\mathbf{c} \in \{0,1\}^{D}$, a one-hot indicator of the operational regime (e.g., a task identifier, operational mode flag, or hardware diagnostic signal) over $D$ discrete contexts. A centralized gating network $\mathbf{W}_g \in \mathbb{R}^{E \times D}$ maps this context signal to a unified macro-routing distribution over the expert pool:}
\begin{equation}
    \rev{g(\mathbf{c}) = \text{Softmax}(\mathbf{W}_g \mathbf{c}) \in \mathbb{R}^E.}
\end{equation}

Each expert $e \in \{1, \dots, E\}$ maintains its own complete $K_y \times K_x$ spatial grid of sub-block adapters. The final, block-wise routed update to output block $i$ linearly scales the structural transformations of all experts:
\begin{equation}
    \mathbf{h}_{i}^{\text{Block}} = \sum_{e=1}^E \text{\del{$g_e(\mathbf{x})$}}\,\rev{g_e(\mathbf{c})} \left( \sum_{j=1}^{K_x} \mathbf{B}_{i,j}^{(e)} \mathbf{A}_{i,j}^{(e)\top} \mathbf{x}_j \right).
\end{equation}
where $\mathbf{A}_{i,j}^{(e)}$ and $\mathbf{B}_{i,j}^{(e)}$ are the low-rank parameter pairs assigned to coordinate $(i,j)$ within the structural layer of expert $e$. This configuration provides high macro-cohesion, dynamically adjusting the full face of the weight update matrix simultaneously based on global input traits.

\subsection{Proposal II: Cell-Wise LoRA-MoE}

To completely bypass single-point routing dependencies and allow fine-grained adaptation across localized coordinates, we propose a decentralized alternative: Cell-Wise LoRA-MoE. Instead of relying on a centralized gating network, this architecture embeds independent micro-routing blocks directly inside every individual coordinate cell $(i, j)$ of the structural grid.

Each coordinate cell maintains a local pool of $E$ low-rank experts $\big\{(\mathbf{B}_{i,j}^{(e)}, \mathbf{A}_{i,j}^{(e)})\big\}_{e=1}^E$. \del{The cell-specific micro-router $\mathbf{W}_g^{(i,j)} \in \mathbb{R}^{E \times (d_{\text{in}}/K_x)}$ generates gating activations conditioned exclusively on its localized slice of the input vector $\mathbf{x}_j$:}

\rev{The cell-specific micro-router $\mathbf{W}_g^{(i,j)} \in \mathbb{R}^{E \times D}$ generates an independent gating distribution for its own coordinate cell, conditioned on the same observed context signal $\mathbf{c}$:}
\begin{equation}
    \rev{g^{(i,j)}(\mathbf{c}) = \text{Softmax}\left(\mathbf{W}_g^{(i,j)} \mathbf{c}\right) \in \mathbb{R}^E.}
\end{equation}

The final conditional output contribution of cell $(i, j)$ is computed autonomously inside its coordinate boundaries. The full update for output block $i$ is formulated as:
\begin{equation}
    \mathbf{h}_{i}^{\text{Cell}} = \sum_{j=1}^{K_x} \left( \sum_{e=1}^E \text{\del{$g_e^{(i,j)}(\mathbf{x}_j)$}}\,\rev{g_e^{(i,j)}(\mathbf{c})} \mathbf{B}_{i,j}^{(e)} \mathbf{A}_{i,j}^{(e)\top} \mathbf{x}_j \right).
\end{equation}

\rev{\paragraph{Routing signal, training, and inference.} In both architectures, the gates are trained end-to-end by backpropagation through the task loss with no auxiliary routing objective. At inference the softmax produces a deterministic soft mixture over the $E$ experts (no sampling or top-$k$ truncation). We condition the gates on the observed context indicator $\mathbf{c}$ instead of on the raw activations. In our experiments, the input distribution is identical across contexts by construction, so context identity is not inferable from activations alone. We verify this empirically through an ablation study on the routing mechanisms (Section~\ref{sec:routing-ablation}). Activation-inferred routing, as required when no explicit context signal is available, is viable only when inputs carry context information and is discussed as future work. 
}

\subsection{Complete Layer Aggregation}

The total forward pass of the reparameterized layer assembles the frozen base mapping with the structural routed outputs. For an arbitrary configuration, the complete output is compiled as:
\begin{equation}
    \mathbf{h} = \mathbf{W}_0 \mathbf{x} + \begin{bmatrix} \mathbf{h}_1^\phi \\ \mathbf{h}_2^\phi \\ \vdots \\ \mathbf{h}_{K_y}^\phi \end{bmatrix}, \quad \phi \in \{\text{Block}, \text{Cell}\}.
\end{equation}

By separating parameter updates across cell regions and introducing coordinate gating pathways, this method ensures that off-diagonal field cross-talk ($\mathbf{B}_{i,j}\mathbf{A}_{i,j}^\top$ for $i \neq j$) can be selected dynamically. Consequently, local subregions can adjust to streaming context switches without overwriting shared parameter boundaries or suffering from gradient cancellation.

\begin{figure}[htbp]
    \centering
    \includegraphics[width=0.95\textwidth]{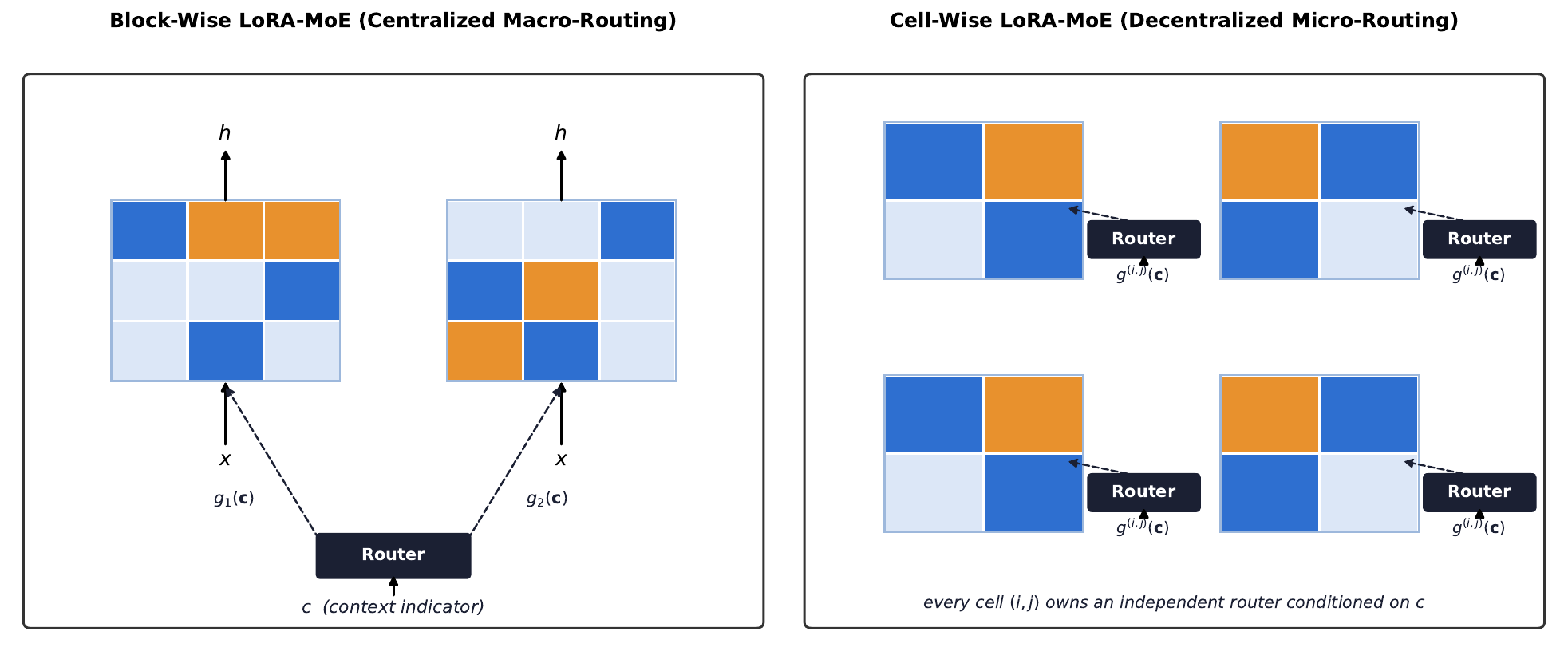}
    \caption{Schematic of the proposed architectures\rev{, matching Eqs.~(5) and~(7) respectively. The input $\mathbf{x}$ feeds the expert grids, while the routers consume the observed context indicator $\mathbf{c}$ and emit the gating mixtures (dashed)}. (Left) block-wise LoRA-MoE utilizes a centralized macro-router to modulate full-grid structural experts, whereas (right) cell-wise LoRA-MoE \rev{embeds} completely decentralized micro-routers within individual sub-block boundaries\rev{, one per coordinate cell $(i,j)$}.}
    \label{fig:design}
\end{figure}

\section{Experiments}
\label{sec:exp}

To evaluate the effectiveness of Localized LoRA-MoE, we conduct a series of empirical investigations across three distinct structural transformation settings: a high-dimensional continuous matrix simulation (Section~\ref{sec:svd}), a real-world multi-field tabular dataset (Section~\ref{sec:housing_data}), and a spatial computer vision perception tracking task under dynamic context shifts \rev{ (Section~\ref{sec:mnist}), followed by fine-tuning experiments language models (Section~\ref{sec:llm})}. The overarching objective is to determine whether distributing parameter-efficient adapter updates across localized block matrix grids—governed by either macro-level routing architectures (\textbf{Localized LoRA-MoE (Global)}) or micro-level routing architectures (\textbf{Localized LoRA-MoE (Cell-Wise)})—can successfully break the \del{optimization ceiling imposed by cross-domain gradient warfare while maintaining strict parameter parity bounds against standard baselines.} \rev{performance ceiling imposed by cross-context gradient interference while staying within the same adapter budgets as baseline methods. 
}

\subsection{Continuous SVD Multi-Field Interaction Simulation}
\label{sec:svd}

This experiment evaluates the framework's ability to navigate multi-context structural transitions over high-dimensional input spaces without succumbing to cross-domain gradient destruction or architectural capacity starvation. To simulate complex real-world data environments, the underlying transformations are generated via continuous matrix projections parameterized through Singular Value Decomposition (SVD), yielding dense, non-trivial statistical properties.

\subsubsection{Problem Formulation and Input Space Partitioning}
We instantiate a mapping system transforming an input vector $\mathbf{x}$ to a target representation $\mathbf{y}$. The input space is defined as a vector $\mathbf{x} \in \mathbb{R}^{d_{\text{in}}}$ with dimensionality $d_{\text{in}} = 64$. The input space is partitioned into $K = 4$ distinct, contiguous feature fields, each of dimension $d_k = 16$:
\begin{equation}
    \mathbf{x} = \begin{bmatrix} \mathbf{x}_1^T & \mathbf{x}_2^T & \mathbf{x}_3^T & \mathbf{x}_4^T \end{bmatrix}^T, \quad \mathbf{x}_i \in \mathbb{R}^{16}
\end{equation}
This configuration models downstream machine learning pipelines—such as a cybersecurity anomaly engine—where an incoming log vector aggregates separate semantic domains, including user profile details, network telemetry streams, and host process events.

The transformation mapping $\mathbf{x} \to \mathbf{y} \in \mathbb{R}^{64}$ is governed by a domain-specific target matrix $\mathbf{W}_{\text{target}}^{(d)}$ which shifts dynamically based on an environmental context variable $d \in \{0, 1, 2\}$:
\begin{equation}
    \mathbf{y} = \mathbf{W}_{\text{target}}^{(d)} \cdot \mathbf{x}
\end{equation}

\subsubsection{Context Generation and Structural Topologies}
The simulation establishes $D=3$ distinct target configurations. To inject mathematical realism, every active sub-block within these target matrices is independently synthesized via SVD to form dense, continuous linear systems:
\begin{equation}
    \mathbf{W}_{\text{block}}^{(i,j)} = \mathbf{U} \mathbf{\Sigma} \mathbf{V}^T
\end{equation}
where $\mathbf{U} \in \mathbb{R}^{16 \times 16}$ and $\mathbf{V} \in \mathbb{R}^{16 \times 16}$ are random orthogonal matrices drawn from a Haar distribution, and $\mathbf{\Sigma} \in \mathbb{R}^{16 \times 16}$ represents a singular value spectrum scaled by $\frac{1}{\sqrt{d_k}}$ to preserve variance bounds across transformations.

We experiment with three correlation structures across LoRA feature blocks, called contexts:
\begin{enumerate}[leftmargin=*,nolistsep]
    \item \textbf{Context 0: Intra-Field Conditioning (Block-Diagonal Profile)} \\
    The system behavior relies strictly on self-contained features. Interactions between disparate blocks are mathematically irrelevant, yielding a ground-truth matrix structured as:
    \begin{equation}
        \mathbf{W}_{\text{target}}^{(0)} = \begin{bmatrix} 
        \mathbf{W}_{\text{block}}^{(1,1)} & \mathbf{0} & \mathbf{0} & \mathbf{0} \\
        \mathbf{0} & \mathbf{W}_{\text{block}}^{(2,2)} & \mathbf{0} & \mathbf{0} \\
        \mathbf{0} & \mathbf{0} & \mathbf{W}_{\text{block}}^{(3,3)} & \mathbf{0} \\
        \mathbf{0} & \mathbf{0} & \mathbf{0} & \mathbf{W}_{\text{block}}^{(4,4)}
        \end{bmatrix}
    \end{equation}

    \item \textbf{Context 1: Inter-Field Cross-Talk (Checkerboard Pattern)} \\
    The system state triggers structural dependencies between disjoint feature domains. This populates off-diagonal blocks alternating on odd coordinate sums satisfying $(i+j) \pmod 2 = 1$:
    \begin{equation}
        \mathbf{W}_{\text{target}}^{(1)} = \begin{bmatrix} 
        \mathbf{0} & \mathbf{W}_{\text{block}}^{(1,2)} & \mathbf{0} & \mathbf{W}_{\text{block}}^{(1,4)} \\
        \mathbf{W}_{\text{block}}^{(2,1)} & \mathbf{0} & \mathbf{W}_{\text{block}}^{(2,3)} & \mathbf{0} \\
        \mathbf{0} & \mathbf{W}_{\text{block}}^{(3,2)} & \mathbf{0} & \mathbf{W}_{\text{block}}^{(3,4)} \\
        \mathbf{W}_{\text{block}}^{(4,1)} & \mathbf{0} & \mathbf{W}_{\text{block}}^{(4,3)} & \mathbf{0}
        \end{bmatrix}
    \end{equation}

    \item \textbf{Context 2: Asymmetric Causal Cascading (Upper Off-Diagonal Shift)} \\
    The system experiences a directional causal chain where feature field $i$ drives targets in field $i+1$. The parameters are confined to the localized upper off-diagonal layer:
    \begin{equation}
        \mathbf{W}_{\text{target}}^{(2)} = \begin{bmatrix} 
        \mathbf{0} & \mathbf{W}_{\text{block}}^{(1,2)} & \mathbf{0} & \mathbf{0} \\
        \mathbf{0} & \mathbf{0} & \mathbf{W}_{\text{block}}^{(2,3)} & \mathbf{0} \\
        \mathbf{0} & \mathbf{0} & \mathbf{0} & \mathbf{W}_{\text{block}}^{(3,4)} \\
        \mathbf{0} & \mathbf{0} & \mathbf{0} & \mathbf{0}
        \end{bmatrix}
    \end{equation}
\end{enumerate}
\rev{To differentiate the signal magnitude of the three regimes, the active blocks of Contexts~1 and~2 below carry additional scale factors of $0.8$ and $1.1$, respectively.} 

\subsubsection{\del{The Gradient Warfare Mechanism} \rev{The Gradient Interference Mechanism}}
To ensure that the learning challenge targets the representational agility of the adaptation layer, the models are evaluated on a frozen base matrix initialized close to zero: $\mathbf{W}_0 \sim \mathcal{N}(0, 0.01^2)$. This structure prevents the base network from buffering multi-domain properties, forcing the full optimization load onto the trainable adapter updates ($\Delta \mathbf{W}$).

The dataset stream samples $N = 10,000$ continuous vectors $\mathbf{x} \sim \mathcal{N}(0, \mathbf{I}_{64})$. Crucially, the context tracker $d$ is randomized uniformly across the sequence using an index generator: $d_n \sim \mathcal{U}\{0, D-1\}$, which models multi-modal pipeline environments where operational states change from one processing batch to the next.

Because static architectures utilize a singular parameter block to accommodate all incoming data, the updates backpropagated through consecutive steps experience destructive cross-cancellation. Step $n$ (Context 0) issues gradients attempting to zero out off-diagonal weights to preserve diagonal tracking, while step $n+1$ (Context 1) maps dense coefficients to those identical off-diagonal coordinates, generating a geometric conflict:
\begin{equation}
    \langle \nabla_{\mathbf{\Theta}} \mathcal{L}(d_n), \, \nabla_{\mathbf{\Theta}} \mathcal{L}(d_{n+1}) \rangle < 0
\end{equation}
This gradient deadlock forces monolithic, unrouted networks to stall during backpropagation, causing them to converge to an unspecialized average of all three matrices.

\subsubsection{Baseline and Proposed Configurations}
To validate performance under rigid optimization conditions, all models are constrained to a uniform parameter footprint of $\sim 2,048$ parameters:
\begin{itemize}[leftmargin=*,nolistsep]
    \item \textbf{LoRA:} Instantiates standard low-rank matrices across the full layer dimensions with rank $r=16$. Trainable Parameters: $2 \times 64 \times 16 = \mathbf{2,048}$.
    \item \textbf{MELoRA \citep{ren2024melora}:} Restricts updates purely to four isolated diagonal sub-blocks using a diagonal low-rank allocation ($r_{\text{diag}}=16$). Trainable Parameters: $4 \times (2 \times 16 \times 16) = \mathbf{2,048}$.
    \item \textbf{Localized LoRA \citep{barazandeh2025localized}:} Distributes parameters across a full $4 \times 4$ block grid ($r_{\text{block}}=4$) but lacks contextual switching capabilities. Trainable Parameters: $16 \times (2 \times 16 \times 4) = \mathbf{2,048}$.
    \item \textbf{Proposed Localized LoRA-MoE (Global):} Integrates $M=2$ full-grid expert layers governed by a monolithic context-conditioned gating projection $\mathbf{W}_g \in \mathbb{R}^{2 \times 3}$. Trainable Parameters: $2 \times [16 \times (2 \times 16 \times 2)] + 6 = \mathbf{2,054}$.
    \item \textbf{Proposed Localized LoRA-MoE (Cell-Wise):} Maintains $M=2$ low-rank expert paths per block, but assigns an independent routing tensor $\bm{\mathcal{W}}_g^{(i,j)} \in \mathbb{R}^{3 \times 2}$ to every coordinate cell $(i,j)$ in the matrix grid. Trainable Parameters: $2 \times [16 \times (2 \times 16 \times 2)] + (4 \times 4 \times 3 \times 2) = \mathbf{2,144}$.
    \item \rev{\textbf{LoRA (context concatenation):} As a non-routed baseline, we implement plain LoRA ($r=16$) over the augmented input $[\mathbf{x};\mathbf{c}]$, receiving exactly the same context signal as the routed models. Trainable Parameters: $16 \times (64+3) + 64 \times 16 = \mathbf{2,096}$.}
    \item \rev{\textbf{Oracle switch ($D$ grids):} $D=3$ independent static localized grids ($r_{\text{block}}=1$), hard-selected by the context label with no learned router. Per-context parameters are disjoint, so no sharing across contexts is possible. Trainable Parameters: $3 \times [16 \times (2 \times 16 \times 1)] = \mathbf{1,536}$.}
\end{itemize}

\subsubsection{\del{Quantitative Performance}\rev{Results} and Discussion}
We optimize all networks using AdamW with a learning rate of $1 \times 10^{-3}$ over 30 epochs, \rev{batch size 64, and a 80/20 train/test split. We run every configuration with 10 independent random seeds, and report all results as mean $\pm$ standard deviation over seeds.} Performance is quantified evaluating the final \rev{held-out} Mean Squared Error (MSE) and the Coefficient of Determination ($R^2$ Score), which measures the total percentage of target transformation variance explained by the fine-tuning adapter.

\rev{To measure gradient interference, we report a \emph{gradient-alignment statistic} (grad-cos): at the end of every training epoch we evaluate the adapter-parameter gradient on a fixed probe batch from each context and compute the mean pairwise cosine similarity across contexts averaged over epochs and seeds. Values near $-1$ indicate destructive cross-context interference and values near $0$ indicate decoupled updates.}

\del{Table 1 previously reported single-seed results: MELoRA MSE 0.9525 / $R^2$ 10.00\%; Localized LoRA 0.8401 / 20.62\%; LoRA 0.8390 / 20.72\%; Localized LoRA-MoE (Global) 0.7746 / 26.81\%; Localized LoRA-MoE (Cell-Wise) 0.6531 / 38.29\%. It is replaced by the multi-seed table below.}

\begin{table}[htbp]
\centering
\setlength{\tabcolsep}{4pt} 
\caption{\rev{Performance comparison on synthetic data}}
\label{tab:numerical_results}
\vspace{2mm}
\rev{\begin{tabular}{lcccc}
\toprule
\textbf{Method} & \textbf{\shortstack{Trainable\\Parameters}} & \textbf{\shortstack{Evaluation\\MSE}} & \textbf{\shortstack{Variance Expl.\\\% ($R^2$)}} & \textbf{grad-cos} \\
\midrule
MELoRA & 2,048 & $0.951 \pm 0.023$ & $9.91 \pm 1.44$ & $-0.256$ \\
Localized LoRA & 2,048 & $0.843 \pm 0.018$ & $20.11 \pm 0.54$ & $-0.328$ \\
LoRA & 2,048 & $0.840 \pm 0.019$ & $20.43 \pm 0.50$ & $-0.243$ \\
LoRA (context-concat) & 2,096 & $0.840 \pm 0.019$ & $20.42 \pm 0.54$ & $-0.243$ \\
Oracle switch ($D$ grids) & 1,536 & $0.839 \pm 0.015$ & $20.53 \pm 0.87$ & $\phantom{-}0.000$ \\
\textbf{Localized LoRA-MoE (Global)} & \textbf{2,054} & $\mathbf{0.770 \pm 0.017}$ & $\mathbf{27.08 \pm 0.74}$ & $\mathbf{-0.252}$ \\
\textbf{Localized LoRA-MoE (Cell-Wise)} & \textbf{2,144} & $\mathbf{0.665 \pm 0.013}$ & $\mathbf{36.98 \pm 0.68}$ & $\mathbf{-0.218}$ \\
\bottomrule
\end{tabular}}
\end{table}

The results in Table~\ref{tab:numerical_results} offer a number of insights about the methods compared.

\begin{enumerate}[leftmargin=*,nolistsep]
    \item \textbf{The Failure Mode of Diagonal Enforcements:} MELoRA encounters significant optimization limits, explaining only \del{$10.00\%$} \rev{$\sim 10\%$} of the data variance. Because its structural constraints force off-diagonal elements to remain strictly zero, it encounters an absolute architectural barrier when handling the inter-field crosstalk found in Contexts 1 and 2.
    \item \textbf{The Static Ceiling:} LoRA (\del{$20.72\%$} and Static Localized LoRA (\del{$20.62\%$} \rev{$\sim 20\%$}) hit \del{an identical} \rev{a similar} performance ceiling. Although Static Localized LoRA contains the spatial blocks necessary to map out off-diagonal cross-talk, it lacks context awareness. The alternating domain stream subjects its parameters to continuous \del{gradient warfare} \rev{gradient interference}, causing them to collapse into an unspecialized average matrix.
    \item \textbf{The Cell-Wise Spatial Routing \del{Breakthrough}:} The proposed localized routing layouts overcome these boundaries. The \rev{block-wise variant using a} macro-routing layout\del{, \textit{Proposed Localized LoRA-MoE (Global)}, } achieves an $R^2$ score of \del{$26.81\%$} \rev{$\sim 27\%$} \del{by switching entire grid updates monolithically. However, }\rev{, while our cell-wise }decentralized variant\del{, \textbf{Localized LoRA-MoE (Cell-Wise)}, } outperforms all alternatives, explaining \del{\textbf{38.29\%}} \rev{$\sim 37\%$} of the overall variance\del{---a $+11.48\%$ absolute boost over the global routing alternative}.
\end{enumerate}

By assigning a dedicated micro-router to each individual block coordinate cell $(i,j)$, the network removes the constraint of homogeneous gating choices. During inference, a diagonal sub-block cell can route inputs to an expert optimized for localized properties, while an off-diagonal cell pivots its local expert mix to capture crosstalk between different modes. This micro-routing framework adds a tiny footprint of just 96 parameters\del{, yet it provides a robust gradient firewall that allows localized low-rank spaces to cleanly converge without cross-domain parameter corruption.}\rev{, allowing localized low-rank subspaces to specialize per coordinate. Note that the grad-cos values of the routed models on this benchmark remain comparable to the static baselines. With 3 contexts but only 2 experts, the expert pool cannot fully separate the regimes and residual interference persists. This is exactly the under-provisioned setting of Theorem~\ref{thm:separation} later, whose error-floor analysis predicts that neither architecture can drive the error to zero here. Expert-utilization statistics confirm healthy specialization without collapse: under context routing, the global gate assigns sharply different expert mixtures to different contexts (e.g., expert-0 mass of $0.11-0.94$ across contexts), while cell-wise gates specialize heterogeneously per coordinate.

The last two baselines separate the value of the routing architecture from mere access to the context signal. Simply appending the label to the input performs identically to plain LoRA, indicating that label access without conditional structure is not useful. \emph{Oracle switching}---which uses the label perfectly but holds $D=3$ disjoint per-context grids---achieves and $R^2$ of only $\sim 20 \%$, similar to LoRA. With three contexts splitting the budget, each disjoint grid is too small to capture its topology, whereas the routed models mix 2 \emph{shared} expert grids into distinct per-context combinations.

\subsubsection{Ablation: Activation-conditioned Routing}
\label{sec:routing-ablation}
To isolate the role of the routing signal, we repeat the analysis with gates conditioned on the activations instead of the context indicator. That is, the global gate reads the full input $\mathbf{x}$ ($\mathbf{W}_g \in \mathbb{R}^{E \times d_{\text{in}}}$, 128 router parameters) and each cell gate reads its local slice $\mathbf{x}_j$ ($\mathbf{W}_g^{(i,j)} \in \mathbb{R}^{E \times d_k}$, 512 router parameters in total). Both routed variants collapse to baseline performance ($19.98 \pm 0.54\%$ and $19.94 \pm 0.48\%$ $R^2$, versus $20.45 \pm 0.50\%$ for plain LoRA), with every gate pinned at the uniform mixture ($0.500 \pm 0.002$) across all contexts and seeds. This outcome is dictated by the benchmark design: the input distribution is identical across contexts ($\mathbf{x} \sim \mathcal{N}(0, \mathbf{I}_{64})$ for every $d$), so context identity carries zero mutual information with the input and no activation-conditioned gate can recover it. The ablation establishes that the gains in Table~\ref{tab:numerical_results} come from conditional routing on an informative context signal, and delineates the practical requirement for activation-inferred routing: the inputs themselves must be informative of the context, as is expected in language model fine-tuning.}

\subsection{Cross-Domain Adaptation in Tabular Data}
\label{sec:housing_data}

To validate the structural and optimization properties of the proposed architectures in a realistic setting, we construct a controlled numerical benchmark utilizing the \textbf{California Housing} dataset~\citep{pedregosa2011scikit}. This evaluation isolates the framework's capacity to navigate context-driven structural transitions without succumbing to cross-domain gradient destruction or architectural capacity starvation.

\subsubsection{Problem Formulation and Multi-Field Feature Partitioning}
We instantiate a mapping system that transforms an incoming real-world feature vector $\mathbf{x} \in \mathbb{R}^{d_{\text{in}}}$ to a target representation $\mathbf{y} \in \mathbb{R}^{d_{\text{out}}}$, where $d_{\text{in}} = d_{\text{out}} = 8$. Rather than treating the feature space as an unstructured sequence, we partition the 8 census attributes into $K = 2$ distinct semantic feature fields of dimension $d_k = 4$:
\begin{equation}
    \mathbf{x} = \begin{bmatrix} \mathbf{x}_{\text{demographics}}^T & \mathbf{x}_{\text{property}}^T \end{bmatrix}^T
\end{equation}
where $\mathbf{x}_{\text{demographics}}$ aggregates socioeconomic indicators (\texttt{MedInc}, \texttt{HouseAge}, \texttt{Population}, \texttt{AveOccup}) and $\mathbf{x}_{\text{property}}$ aggregates physical asset profiles (\texttt{AveRooms}, \texttt{AveBedrms}, \texttt{Latitude}, \texttt{Longitude}).

The mapping pipeline simulates a multi-context downstream system. The raw input vector $\mathbf{x}$ is scaled to standard normal distributions via a localized transformation layer $\mathbf{x} \sim \mathcal{N}(\bm{\mu}, \SigmaDiag)$. Here, $\SigmaDiag$ is a block-diagonal covariance matrix that ensures each feature field is normalized relative only to its internal statistics, thereby preserving the structural independence of the input domains prior to adapter processing.

Given this partitioned input space, the operational mapping $\mathbf{x} \to \mathbf{y}$ is conditioned on an environmental binary context variable $c \in \{0, 1\}$ passed alongside the batch:
\begin{equation}
    \mathbf{y} = \mathbf{W}_{\text{target}}^{(c)} \cdot \mathbf{x}
\end{equation}
\subsubsection{Context Generation and Matrix Topologies}
The transformation matrices are synthesized using orthogonal base mappings extracted via SVD to prevent trivial geometric arrangements. The context variable $c$ dictates how the socioeconomic blocks and property/geographic blocks cross-correlate, generating two distinct structural topologies:

\begin{enumerate}[leftmargin=*,nolistsep]
    \item \textbf{Context 0: Intra-Field Dominant Transformation ($c=0$)} \\
    The evaluation dictates that features only interact within their own structural boundaries. This yields a strictly \textbf{block-diagonal} target transformation matrix:
    \begin{equation}
        \mathbf{W}_{\text{target}}^{(0)} = \begin{bmatrix} 
        \mathbf{W}_{\text{block}}^{(1,1)} & \mathbf{0} \\
        \mathbf{0} & \mathbf{W}_{\text{block}}^{(2,2)}
        \end{bmatrix}
    \end{equation}
    where each active sub-block $\mathbf{W}_{\text{block}} \in \mathbb{R}^{4 \times 4}$ represents a dense SVD system scaled by $0.5$ to ensure bounded stability.

    \item \textbf{Context 1: Inter-Field Cross-Talk Transformation ($c=1$)} \\
    The system shifts into a cross-dependent state where target profiles are exclusively driven by the cross-talk \emph{between} disjoint spaces. This results in an \textbf{off-diagonal} target profile:
    \begin{equation}
        \mathbf{W}_{\text{target}}^{(1)} = \begin{bmatrix}
        \mathbf{0} & \mathbf{W}_{\text{block}}^{(1,2)} \\
        \mathbf{W}_{\text{block}}^{(2,1)} & \mathbf{0}
        \end{bmatrix}
    \end{equation}
    \rev{where the active off-diagonal sub-blocks are dense SVD systems scaled by $0.4$.}
\end{enumerate}

\rev{We construct each active $4\times4$ block as $\mathbf{U}\mathbf{\Sigma}\mathbf{V}^\top$ with Haar-random orthogonal factors, and the context of each sample is drawn as $c \sim \mathcal{U}\{0,1\}$ independently of the features. We verify that all conclusions below are unchanged when the blocks are instead drawn as plain Gaussian matrices, so the results are robust to the target construction.This acts as a controlled context-switching testbed built on a real-world dataset, designed to show that the interference phenomenon and its resolution by routing persist beyond purely synthetic inputs.}

\subsubsection{The Gradient \del{Warfare}\rev{Interference} Bottleneck}
We train all parameters against a frozen base matrix initialized at $\mathbf{W}_0 \sim \mathcal{N}(0, 0.01^2)$. This structure prevents the base network from absorbing multi-domain properties, forcing the full optimization load onto the adapter updates ($\Delta \mathbf{W}$).

The dataset contains $N = 20,640$ scaled vectors in a randomized sequence of contexts $c_n \sim \mathcal{U}\{0, 1\}$. Because static configurations utilize a singular parameter block to accommodate all incoming data, the updates backpropagated through consecutive steps experience destructive cross-cancellation:
\begin{equation}
    \langle \nabla_{\mathbf{\Theta}} \mathcal{L}(c_n=0), \, \nabla_{\mathbf{\Theta}} \mathcal{L}(c_{n+1}=1) \rangle < 0
\end{equation}
Step $n$ demands that the off-diagonal block elements compress to zero to isolate feature lines, while step $n+1$ demands high-magnitude coefficients across those exact coordinates. This gradient conflict causes unrouted networks to stall during backpropagation, causing them to converge to an unspecialized average of both target spaces.

\subsubsection{Baseline Configurations and Parameter Parity Bounding}
To establish a mathematically rigorous environment, all baselines are tuned relative to their grid dimensions to operate within an identical range of $\sim 256$--$272$ trainable parameters:
\begin{itemize}[leftmargin=*,nolistsep]
    \item \textbf{LoRA ($r=16$):} Applies a singular global low-rank update matrix across the full face of the layer. Trainable Parameters: $2 \times 8 \times 16 = \mathbf{256}$.
    \item \textbf{MELoRA ($num\_blocks=2, r_{\text{diag}}=16$):} Restricts parameter updates purely to two isolated $4 \times 4$ sub-blocks along the main diagonal. Trainable Parameters: $2 \times (2 \times 4 \times 16) = \mathbf{256}$.
    \item \textbf{Localized LoRA Static ($K=2, r_{\text{block}}=8$):} Instantiates a full $2 \times 2$ block grid across the surface, giving every sub-block a static matrix rank of 8, but lacks contextual switching capabilities. Trainable Parameters: $4 \text{ blocks} \times (2 \times 4 \times 8) = \mathbf{256}$.
    \item \textbf{Localized LoRA-MoE (Global) ($M=2, r_{\text{block}}=4$):} Deploys $M=2$ full-grid expert layers governed by a monolithic context gating projection $\mathbf{W}_g \in \mathbb{R}^{2 \times 2}$. Trainable Parameters: $[2 \text{ experts} \times 4 \text{ blocks} \times (2 \times 4 \times 4)] + 4 = \mathbf{260}$.
    \item \textbf{Localized LoRA-MoE (Cell-Wise) ($M=2, r_{\text{block}}=4$):} Maintains $M=2$ low-rank expert paths per block, but introduces independent routing tensors $\bm{\mathcal{W}}_g^{(i,j)} \in \mathbb{R}^{2 \times 2}$ assigned to every cell coordinate $(i,j)$ in the matrix grid. Trainable Parameters: $[2 \text{ experts} \times 4 \text{ blocks} \times (2 \times 4 \times 4)] + (2 \times 2 \times 2 \times 2) = \mathbf{272}$.
    \item \rev{We include \textbf{LoRA (context-concat)} (plain LoRA, $r=16$, over $[\mathbf{x};\mathbf{c}]$; $16 \times 10 + 8 \times 16 = \mathbf{288}$ parameters) and \textbf{oracle switching} ($D=2$ disjoint static $2\times2$ grids with $r_{\text{block}}=4$, hard-selected by the label; $2 \times 4 \times (2 \times 4 \times 4) = \mathbf{256}$ parameters).}
\end{itemize}
\subsubsection{\del{Quantitative }Results and \del{Technical }Discussion}
All networks were optimized using the AdamW optimizer with a learning rate of $5\times 10^{-3}$ for 15 epochs using a batch size of 128\rev{, on an 80/20 train/test split, repeated over 10 independent seeds}. The empirical metrics are compiled in Table~\ref{tab:real_data_results}, and their full epoch-by-epoch convergence signatures are plotted in Figure~\ref{fig:cali_convergence}.

\del{Table 2 previously reported single-seed results: MELoRA MSE 0.8781 / $R^2$ $-7.78$\%; Localized LoRA 0.4106 / 49.61\%; LoRA 0.4424 / 45.70\%; Block-Wise LoRA-MoE (Global) 0.0028 / 99.65\%; Cell-Wise LoRA-MoE 0.0040 / 99.51\%. It is replaced by the multi-seed table below.}

\begin{table}[htbp]
\centering
\small
\setlength{\tabcolsep}{4pt}
\caption{\rev{Performance comparisons under matched adapter budgets on the California Housing context-switching benchmark. Mean $\pm$ std over 10 seeds; grad-cos as in Table~\ref{tab:numerical_results}. Training: AdamW, lr $5\times10^{-3}$, 15 epochs, batch 128, 80/20 train/test split.}}
\label{tab:real_data_results}
\vspace{2mm}
\rev{\begin{tabular}{lcccc}
\toprule
\textbf{Method} & \textbf{\shortstack{Trainable\\Parameters}} & \textbf{\shortstack{Evaluation\\MSE}} & \textbf{\shortstack{Variance Exp.\\\% $R^2$}} & \textbf{grad-cos} \\
\midrule
MELoRA & 256 & $0.194 \pm 0.074$ & $10.88 \pm 31.09$ & $-0.743$ \\
Localized LoRA & 256 & $0.128 \pm 0.076$ & $41.81 \pm 31.69$ & $-0.979$ \\
LoRA & 256 & $0.122 \pm 0.065$ & $44.37 \pm 26.68$ & $-0.977$ \\
LoRA (context-concat) & 288 & $0.130 \pm 0.058$ & $40.49 \pm 24.77$ & $-0.884$ \\
Oracle switch ($D$ grids) & 256 & $0.003 \pm 0.002$ & $98.80 \pm 0.91$ & $\phantom{-}0.000$ \\
\textbf{Block-Wise LoRA-MoE} & \textbf{260} & $\mathbf{0.003 \pm 0.001}$ & $\mathbf{98.53 \pm 0.61}$ & $\mathbf{-0.057}$ \\
\textbf{Cell-Wise LoRA-MoE} & \textbf{272} & $\mathbf{0.006 \pm 0.004}$ & $\mathbf{97.51 \pm 1.83}$ & $\mathbf{-0.067}$ \\
\bottomrule
\end{tabular}}
\end{table}

\begin{figure}[htbp]
    \centering
    \includegraphics[width=0.75\linewidth]{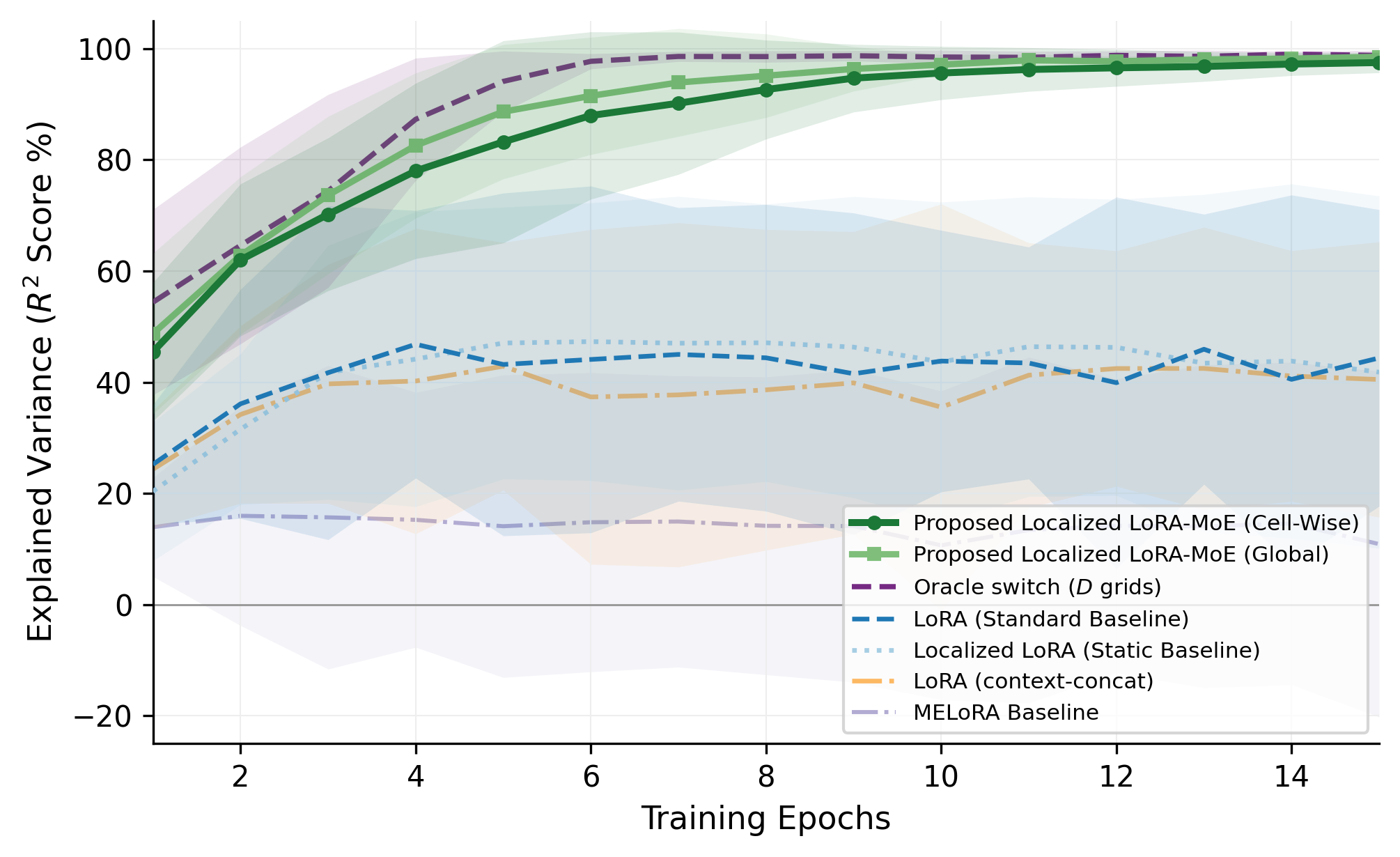}
    \caption{Convergence trajectory of \rev{held-out} explained variance ($R^2$ score \%) over 15 training epochs on the California Housing dataset}
    \label{fig:cali_convergence}
\end{figure}

Among the baseline methods, MELoRA demonstrates \del{an absolute optimization failure, yielding a negative $R^2$ score of $-7.78\%$} \rev{a severe optimization failure, yielding a low $R^2$ with a wide range of errors}. Because its structural layout forces all updates to align strictly along independent diagonal sub-blocks, it is mathematically incapable of representing feature cross-talk. When the system shifts to Context 1, MELoRA has zero parameter allocation over off-diagonal fields, resulting in destructive gradient interference during backpropagation. Both LoRA \del{($45.70\%$} and Static Localized LoRA \del{($49.61\%$} converge to \del{an identical} \rev{a similar} performance plateau. Although LoRA possesses full rank coverage and Static Localized LoRA contains the necessary off-diagonal blocks to process feature cross-talk, neither contains dynamic context routing logic. Lacking an adaptive gating mechanism, static models collapse into an unspecialized average matrix of the system, leaving roughly half of the underlying variance unmodeled due to \del{continuous gradient warfare.} \rev{gradient interference. Table~\ref{tab:real_data_results} shows that the static models' cross-context adapter gradients are near-antipodal (grad-cos $\approx -0.98$), the sharpest interference measured across our benchmarks. The static baselines also exhibit very large seed variance ($\pm 25$--$32$\%), illustrating that single-seed results are unreliable in interference-dominated regimes.}

The introduction of a dynamic contextual gate breaks this \del{optimization} \rev{performance} ceiling, with both of our proposed architectures explaining nearly all variance, as illustrated in Figure~\ref{fig:cali_convergence}\rev{, and reducing the interference statistic to near-orthogonality (grad-cos near zero)}. Because environmental shifts uniformly alter the macro weight profile across tabular blocks, \del{block-wise Localized LoRA-MoE achieves a marginal lead ($99.65\%$ vs $99.51\%$) by maintaining cohesive matrix-wide representations. The results for the cell-wise, decentralized proves} \rev{the two routing granularities show similar mean performance (overlapping standard error bars), exactly as Theorem~\ref{thm:parity} predicts for a routing-uniform target with $E \ge D$, and both match oracle label-switching. This confirms that in this regime the learned gates converge to one-hot per-context selection, while context-concat ($40.49\%$) shows that label access without conditional structure achieves nothing. The cell-wise result shows} that localized block autonomy can reach the performance upper bounds of a global system. By parameterizing independent routing layers at every discrete cell coordinate $(i, j)$, individual block spaces are granted structural autonomy\del{. This creates a robust gradient firewall, allowing specialized sub-blocks to isolate cross-domain parameters cleanly and eliminate gradient overwriting under continuous task shifts.}\rev{, isolating cross-context parameters cleanly and eliminating gradient interference under continuous task shifts.}
\subsection{Robustness Evaluation on Computer Vision Data}
\label{sec:mnist}

To demonstrate the structural utility of our proposal, we evaluate performance under a simulated \emph{Dynamic Sensor Degradation} scenario using the MNIST corpus. This experiment models a production vehicle perception network that must dynamically reorganize its spatial processing weights when an environmental failure occludes a portion of the sensor array.

\subsubsection{Spatial Partitioning and Degradation Topology}
Let an input image canvas be flattened and partitioned into $K = 4$ independent spatial quadrants $\mathbf{x} = [\mathbf{x}_1^T, \mathbf{x}_2^T, \mathbf{x}_3^T, \mathbf{x}_4^T]^T \in \mathbb{R}^{784}$, where each vector component $\mathbf{x}_i \in \mathbb{R}^{196}$ represents a localized $14 \times 14$ pixel window. The network maps these incoming feature spaces to a global target reconstruction tensor $\mathbf{y} \in \mathbb{R}^{784}$ through a transformation matrix $\mathbf{W}_{\text{target}}^{(c)}$ that adapts to an external hardware diagnostic context flag $c \in \{0, 1\}$:
\begin{equation}
    \mathbf{y} = \mathbf{W}_{\text{target}}^{(c)} \cdot \mathbf{x}
\end{equation}

The hardware operational contexts dictate how information must flow across spatial coordinates:
\begin{enumerate}[leftmargin=*,nolistsep]
    \item \textbf{Context 0 (Nominal Array State):} The full camera lens is unobstructed \rev{ and the input image is intact}. The transformation target maps nominal spatial correlations uniformly, relying on dense anti-diagonal block structures to capture balanced symmetric geometric cross-talk across opposite corners of the canvas.
    \item \textbf{Context 1 (Localized Quadrant Failure):} \del{The bottom-right quadrant ($\mathbf{x}_4$) becomes occluded due to physical lens degradation. To maintain system safety, the network must instantly sever dependencies on $\mathbf{x}_4$ and dynamically shift to an asymmetric off-diagonal tracking state, forcing surviving quadrants ($\mathbf{x}_1$ and $\mathbf{x}_3$) to cross-talk directly to reconstruct structural boundaries.} \rev{The bottom-right quadrant ($\mathbf{x}_4$) is occluded \emph{at the input}: a fraction $\rho \in (0,1]$ of its pixels---the degradation \emph{severity}---is zeroed before the sample enters the network. The degraded target depends only on surviving quadrants: $\mathbf{x}_3 \to \mathbf{y}_1$ and $\mathbf{x}_1 \to \mathbf{y}_3$ (surviving-quadrant cross-talk), and $\mathbf{x}_1 \to \mathbf{y}_4$ (reconstruction of the occluded output region from a surviving sensor). All active target blocks are dense Gaussian $196\times196$ matrices scaled by $0.1$.}
\end{enumerate}

\rev{The task is linear spatial reconstruction: predict $\mathbf{y} = \mathbf{W}_{\text{target}}^{(c)} \tilde{\mathbf{x}}$ from the (possibly occluded) input $\tilde{\mathbf{x}}$, with held-out MSE and $R^2 = 1 - \mathrm{MSE}/\mathrm{Var}(\mathbf{y})$ as metrics. We use 20{,}000 MNIST images split 80/20 into train/test, contexts drawn uniformly per sample, a frozen base $\mathbf{W}_0 \sim \mathcal{N}(0, 0.005^2)$, and AdamW with learning rate $2\times10^{-3}$ for 10 epochs at batch size 128, over 3 independent seeds. Within each run, training and evaluation use the same severity; we train and evaluate at both $\rho = 1.0$ and $\rho = 0.5$, and separately evaluate generalization to a severity not seen in training (Section~\ref{sec:unseen-severity}).}

\subsubsection{Gradient Interference in Shared Structural Weights}
Because the network parameterizes a dense $784 \times 784$ transformation surface ($614,656$ base elements), static parameter efficiency is strictly bounded at $\sim 25,000$ trainable units across all baselines. When a static network processes a stream containing interleaved cycles of Nominal and Degraded states, the shared adapter weights experience destructive optimization feedback. The backward path for Context 0 forces updates to maximize symmetric corner correlations, while Context 1 backpropagates updates that suppress those exact pathways to isolate the broken sensor quadrant. Lacking a localized gating mechanism, standard adapters \del{trigger ``gradient warfare''} \rev{incur strong cross-context gradient interference} across the shared parameter block, collapsing into an unspecialized average weight state that cannot adequately resolve either operational reality.

\subsubsection{\del{Quantitative Evaluation}\rev{Results and Discussion}\del{ on Spatial Reconstruction}}
The empirical findings for the high-dimensional spatial vision reconstruction benchmark (Table~\ref{tab:mnist_fault_results}) illustrate distinct geometric bottlenecks when parameter-efficient adapters scale from compact tabular spaces to sprawling spatial coordinate fields.

\del{Table 3 previously reported single-seed results under the earlier protocol (which switched target topologies without occluding the input): MELoRA MSE 0.1281 / $R^2$ 23.56\%; Localized LoRA (Static) 0.1077 / 35.74\%; LoRA 0.1066 / 36.42\%; Proposed (Global Router) 0.0553 / 66.67\%; Proposed (Cell-Wise) 0.0542 / 66.99\%. It is replaced by the corrected, multi-seed table below.}

\begin{table}[htbp]
\centering
\small
\setlength{\tabcolsep}{4pt}
\caption{\rev{Performance of Localized LoRA-MoE and baselines in the MNIST experiment}}
\label{tab:mnist_fault_results}
\rev{\begin{tabular}{lcccc}
\toprule
\textbf{Method} & \textbf{Parameters} & \textbf{MSE} & \textbf{$R^2$ (\%)} & \textbf{grad-cos} \\
\midrule
MELoRA & 50,176 & $0.133 \pm 0.001$ & $23.89 \pm 1.01$ & $-0.944$ \\
Localized LoRA (Static) & 25,088 & $0.070 \pm 0.001$ & $60.04 \pm 0.84$ & $-0.681$ \\
LoRA & 25,088 & $0.071 \pm 0.001$ & $59.19 \pm 0.84$ & $-0.598$ \\
LoRA (context-concat) & 25,120 & $0.068 \pm 0.001$ & $61.07 \pm 0.76$ & $-0.395$ \\
Oracle switch ($D$ grids) & 25,088 & $0.057 \pm 0.001$ & $67.01 \pm 0.73$ & $\phantom{-}0.000$ \\
\textbf{Proposed (Global Router)} & 25,092 & $\mathbf{0.056 \pm 0.001}$ & $\mathbf{68.00 \pm 0.86}$ & $\mathbf{-0.073}$ \\
\textbf{Proposed (Cell-Wise)} & 25,152 & $\mathbf{0.054 \pm 0.001}$ & $\mathbf{68.75 \pm 0.94}$ & $\mathbf{-0.214}$ \\
\bottomrule
\end{tabular}}
\end{table}

\rev{As in the previous experiments, }LoRA \del{($36.42\%$)} and Static Localized LoRA \del{($35.74\%$)} have the worst performance. Lacking conditional execution logic, the optimizer is trapped in an architectural tug-of-war. The resulting gradient interference \rev{(grad-cos $-0.60$ to $-0.68$)} collapses the weight trajectories into an unspecialized mathematical average, failing to protect the perception pipeline from localized sensor dropouts. MELoRA demonstrates severe capacity starvation, capturing only \del{$23.56\%$} \rev{$\sim 24\%$} of the variance despite using a double parameter budget. This exposes the flaw of assuming diagonal block independence in structural vision: by physically prohibiting updates from populating off-diagonal coordinate spaces, it prevents the network from resolving multi-quadrant structural cross-talk required for safe perception.

The injection of contextual gating breaks the gradient interference bottleneck, yielding \del{a $\sim 30\%$ lift in explained variance. Both versions of our proposal--- global ($66.67\%$) and cell-wise ($66.99\%$) architectures successfully resolve the structural bottleneck.} \rev{a lift of approximately $9\%$ in explained variance and reducing the interference statistic from $-0.60$ (LoRA) to $-0.07$ (global router). Both versions of our proposal resolve the structural bottleneck, and are statistically at parity within the training severity, slightly ahead of oracle label-switching and well ahead of context-concatenation. These results are stable across both severities tested ($\rho = 1.0$ and $\rho = 0.5$).}

\begin{figure}[H]
    \centering
    \includegraphics[width=0.75\textwidth]{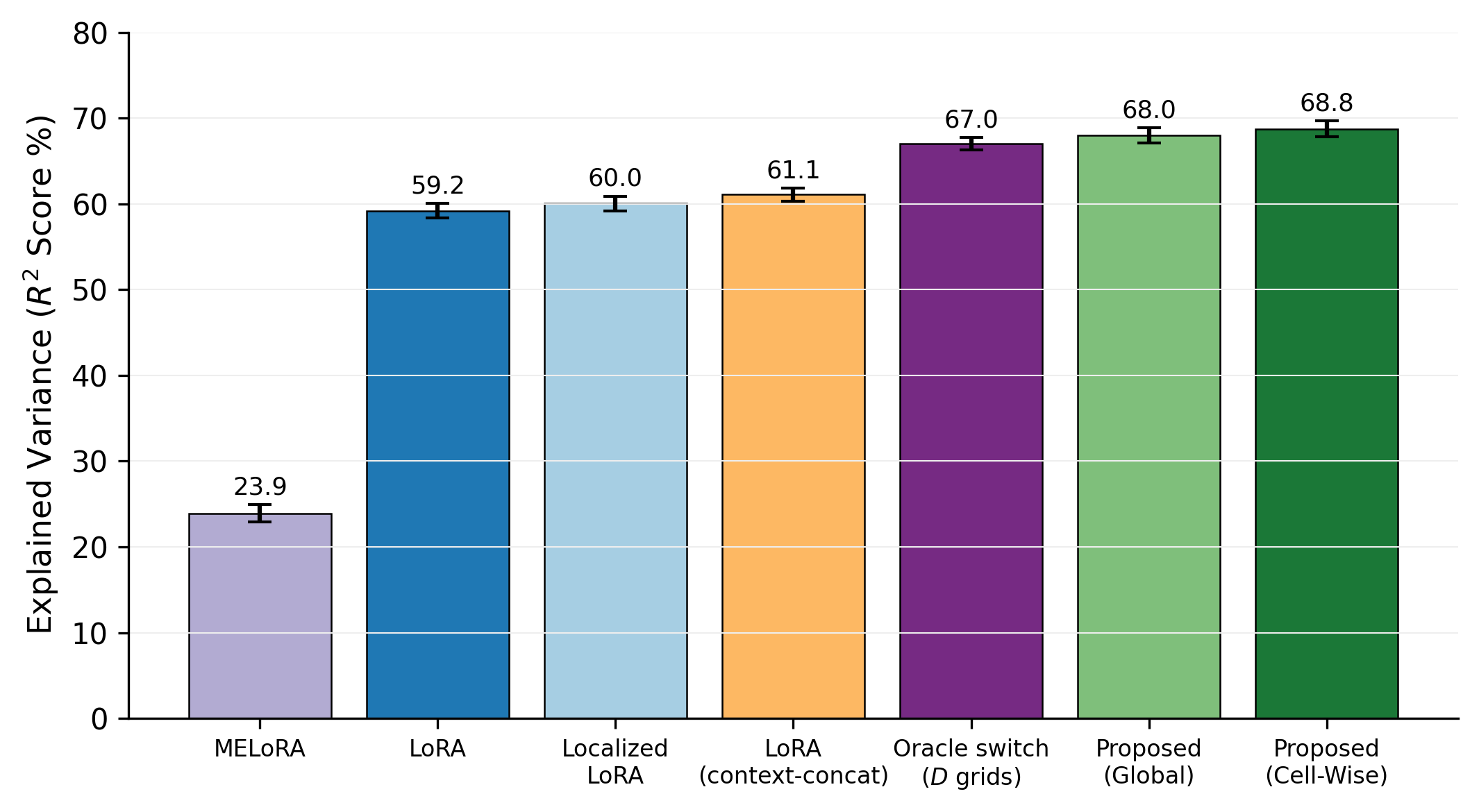}
    \caption{Performance comparison of proposed routing architectures \rev{on the MNIST reconstruction experiment}}
    \label{fig:mnist_comparison}
\end{figure}

\del{This demonstrates that decentralized block-level gating can match the efficacy of an omniscient global coordinator. By granting each localized block cell the structural autonomy to configure its own routing space conditionally, it isolates cross-domain tasks natively. This allows the model to protect surviving sensor channels from fault-propagated gradient corruption without relying on a rigid, single-point global routing bottleneck.}

\rev{\subsubsection{Generalization to Unseen Degradation Severity}
\label{sec:unseen-severity}
We take the models trained at full occlusion ($\rho = 1.0$) and evaluate them at the unseen partial severity $\rho = 0.5$, holding the target matrices and context assignments fixed and using nested occlusion masks, so that the comparison is paired and only the degradation level differs.}

\begin{table}[htbp]
\centering
\small
\caption{\rev{Generalization results from training severity $\rho = 1.0$ to unseen severity $\rho = 0.5$}}
\label{tab:mnist_unseen}
\rev{\begin{tabular}{lcc}
\toprule
\textbf{Method} & \textbf{$R^2$ (\%) at trained $\rho=1.0$} & \textbf{$R^2$ (\%) at unseen $\rho=0.5$} \\
\midrule
MELoRA & $23.89 \pm 1.01$ & $23.13 \pm 1.13$ \\
Localized LoRA (Static) & $60.04 \pm 0.84$ & $49.03 \pm 0.95$ \\
LoRA & $59.19 \pm 0.84$ & $48.27 \pm 1.13$ \\
LoRA (context-concat) & $61.07 \pm 0.76$ & $56.46 \pm 1.05$ \\
Oracle switch ($D$ grids) & $67.01 \pm 0.73$ & $66.97 \pm 0.73$ \\
Proposed (Global Router) & $68.00 \pm 0.86$ & $47.88 \pm 9.58$ \\
\textbf{Proposed (Cell-Wise)} & $\mathbf{68.75 \pm 0.94}$ & $\mathbf{65.63 \pm 0.84}$ \\
\bottomrule
\end{tabular}}
\end{table}

\rev{The cell-wise architecture loses only $3\%$ under the unseen condition, while the global router loses $20\%$, and the static baselines lose $\sim 11\%$. The mechanism follows Corollary~\ref{cor:firewall}. At training severity $\rho = 1.0$ the occluded quadrant contributes no gradient in the degraded context, so the column-4 cells of the degraded pathway remain near their zero initialization; at test time, partially restored pixels leak into the network. Under per-cell gating, the degraded context routes each column-4 cell to its near-zero expert, so the leaked signal is suppressed locally. Under a shared global gate, the degraded-context expert mixture is a single grid-wide decision whose column-4 blocks carry weights shaped by the nominal context, so the leaked pixels propagate through unsuitable pathways across the entire output. Oracle switching is equally robust here. Its degraded-context grid is fully disjoint, obtaining the same containment by construction, but at the cost of failing entirely when contexts outnumber experts. The cell-wise architecture is thus the only configuration in our study that achieves \emph{both} expert sharing (capacity when $D > E$) and fault containment (robustness under unseen degradation).

\subsection{Language Model Fine-tuning}
\label{sec:llm}

We now present experiments on encoder-only and decoder-only language models to demonstrate the practical efficacy of Localized LoRA-MoE.

\subsubsection{Encoder Model}
We fine-tune RoBERTa-base (125M parameters) on six GLUE tasks (RTE, MRPC, SST-2, CoLA, QNLI, STS-B) in a \emph{multi-task interleaved} stream: every training batch belongs to one task, batches from all tasks are shuffled together, and a single adapter set must serve all six. This is the transformer analogue of our context-switched streams, with the observed context indicator $\mathbf{c}$ instantiated as the one-hot task identity ($D = 6$). Adapters are injected into the query and value projections of all 12 layers (24 matrices, $d = 768$). Per-task classification heads are identical across methods and excluded from the adapter budgets below.

All adapter configurations are matched at a total budget of 294{,}912 trainable adapter parameters: LoRA with $r = 8$; static Localized LoRA with $k = 2$, $r_b = 4$ (the structure of GraLoRA~\citep{jung2026gralora}); our two variants with $E = 2$ experts, $k = 2$, $r_b = 2$ and soft task-conditioned gates (router overhead: 288 parameters for block-wise, 1{,}152 for cell-wise). We repeat both routed variants with activation-conditioned gates (per-token $\mathbf{x}_t$, and sequence-pooled $\bar{\mathbf{x}}$) as the routing-signal ablation. Finally, we have two oracle baselines with a higher parameter budget that use task switching, i.e.\ one disjoint localized grid per task, at $r_b = 1$ (442K parameters, $1.5$x our budget) and at $r_b = 4$ (1.77M parameters, $6$x our budget). Training uses AdamW ($2\times10^{-4}$, cosine schedule, 25 epochs, batch 32, up to 4{,}000 examples per task, adapter scaling $2.0$) over 3 seeds.
}

\begin{table}[htbp]
\centering
\small
\setlength{\tabcolsep}{4pt}
\caption{\rev{Multi-task GLUE finetuning performance on RoBERTa-base}}
\label{tab:glue}
\rev{\begin{tabular}{llcc}
\toprule
\textbf{Method} & \textbf{\shortstack{Routing\\signal}} & \textbf{\shortstack{Adapter+Router\\Parameters}} & \textbf{\shortstack{Accuracy (\%)}} \\
\midrule
LoRA & --- & 294,912 & $84.87 \pm 0.28$ \\
Localized LoRA & --- & 294,912 & $85.00 \pm 0.84$ \\
Task switch ($r_b=1$) & task id (hard) & 442,368 & $85.91 \pm 0.14$ \\
Task switch ($r_b=4$) & task id (hard) & 1,769,472 & $85.98 \pm 0.29$ \\
\midrule
\textbf{Localized LoRA-MoE (Global)} & task id, $g(\mathbf{c})$ & 294,912 + 288 & $\mathbf{85.88 \pm 0.28}$ \\
\textbf{Localized LoRA-MoE (Cell-Wise)} & task id, $g(\mathbf{c})$ & 294,912 + 1,152 & $\mathbf{85.78 \pm 0.15}$ \\
Localized LoRA-MoE (Global) & per-token $\mathbf{x}_t$ & 294,912 + 36,864 & $85.36 \pm 0.12$ \\
Localized LoRA-MoE (Cell-Wise) & per-token $\mathbf{x}_t$ & 294,912 + 73,728 & $85.45 \pm 0.18$ \\
Localized LoRA-MoE (Global) & pooled $\bar{\mathbf{x}}$ & 294,912 + 36,864 & $84.16 \pm 0.30$ \\
Localized LoRA-MoE (Cell-Wise) & pooled $\bar{\mathbf{x}}$ & 294,912 + 73,728 & $84.61 \pm 0.51$ \\
\bottomrule
\end{tabular}}
\end{table}

\rev{
Table~\ref{tab:glue} yields four findings. First, task-conditioned routing beats the context-blind baselines at an identical budget, with the gains concentrated on the smaller tasks that suffer most from sharing one static adapter (RTE, STS-B). On a benchmark where full fine-tuning exceeds LoRA by about one point, this captures most of the available headroom. Second, the routed adapters match both oracles, achieving comparable (within standard error bounds) performance with a much lesser parameter budget. Third, per-token activation routing reaches similar performance to task conditioning. Since transformer activations are task-informative; explicit task conditioning still remains best. Sequence-mean pooling performs worst, indicating that pooling destroys task-discriminative information. Fourth, cell-wise and block-wise routing are statistically tied, as Theorem~\ref{thm:parity} predicts: the task identity shifts the adapter uniformly across the grid, the routing-uniform regime in which block-wise and cell-wise routing coincide.

\rev{In terms of compute, one multi-task training run takes 131~min for LoRA, 144~min for the static grid, and $\approx$190~min for both routed variants on identical hardware: soft routing with $E=2$ experts costs about $1.45\times$ LoRA's training wall-clock, the expected price of evaluating both experts on every forward pass, while adding essentially no parameters.}

\subsubsection{Decoder Model}
\label{app:qwen}

We fine-tune Qwen2.5-1.5B on the Commonsense-170K instruction mixture, with adapters injected into all 7 linear projections of all 28 transformer blocks (196 matrices). Training uses the standard recipe for this model: Lion optimizer ($\beta_1 = 0.9$, $\beta_2 = 0.95$), learning rate $2\times10^{-4}$ with cosine schedule and 10\% warmup, batch size 192, 2 epochs, bf16, sequence length 256, and three random seeds. For evaluation, we use the official test sets for the eight commonsense benchmarks, the official prompt template, beam search (4 beams, 32 new tokens) under KV-cached decoding, and the official answer-extraction rules. 

Because no explicit context indicator exists in this deployment, the routers here are \emph{activation-conditioned at the token level}. This regime fails on our synthetic benchmarks (Section~\ref{sec:routing-ablation}) but is viable here because transformer activations carry task information, as established in the GLUE experiment. We compare four methods: LoRA ($r = 64$), static Localized LoRA ($k = 4$, $r_b = 16$), and our two routed variants ($E = 4$ experts, $k = 4$, $r_b = 8$, top-2 gating). LoRA and the static grid use 73.9M parameters, while each routed variant stores four expert grids (147.7M) plus its routers (2.0M block-wise; 8.1M cell-wise), so their active parameter counts are 73.9M plus their respective router sizes.
}

\begin{table}[htbp]
\centering
\small
\setlength{\tabcolsep}{3pt}
\caption{\rev{Commonsense reasoning accuracy (\%) of Qwen2.5-1.5B fine-tuned on Commonsense-170K, split across all tasks. Active parameter count in Millions for each method in parentheses}}
\label{tab:qwen}
\scriptsize
\setlength{\tabcolsep}{2pt}
\rev{\begin{tabular}{lccccccccc}
\toprule
\textbf{Method} & \textbf{BoolQ} & \textbf{PIQA} & \textbf{SIQA} & \textbf{HS} & \textbf{WG} & \textbf{ARC-c} & \textbf{ARC-e} & \textbf{OBQA} & \textbf{Avg} \\
\midrule
LoRA (73.9) & 67.0 & 81.7 & 74.6 & 81.8 & 74.9 & 69.3 & 83.2 & 80.7 & $76.65 \pm 0.30$ \\
Localized LoRA (73.9) & 63.3 & 82.1 & 75.6 & 84.3 & \textbf{75.7} & 73.2 & 87.7 & 82.8 & $78.10 \pm 0.61$ \\
\textbf{Global} (75.9) & 67.8 & 84.0 & \textbf{76.4} & 87.3 & 75.4 & 76.5 & 89.6 & \textbf{84.2} & $80.15 \pm 0.22$ \\
\textbf{Cell-wise} (82.0) & \textbf{68.6} & \textbf{84.4} & 76.0 & \textbf{88.1} & \textbf{75.7} & \textbf{77.9} & \textbf{89.8} & 84.0 & $\mathbf{80.57 \pm 0.24}$ \\
\bottomrule
\end{tabular}}
\end{table}

\rev{
The results in Table~\ref{tab:qwen} indicate that both variants of our proposal continue to perform better than non-routed methods. Static partitioning improves over LoRA by about 1.5\%, and routing over the partitioned grid adds a further 2 to 2.5\% lift, with the cell-wise variant performing better than the global variant. As observed in the Housing data and GLUE experiments, localized LoRA is erratic across seeds (0.61 SE), while the routed variants are stable on the same tasks (0.22 and 0.24 SE). Thus, routing acts as the arbitration mechanism that static partitioning lacks.
}

\section{Theoretical Analysis}
\label{sec:theory}

The empirical results in Section~\ref{sec:exp} show that both of our proposed architectures break the static ceiling. The two are at parity when context shifts the grid uniformly, and cell-wise routing strictly dominates when context reorganizes coordinates
heterogeneously. In this section, we theoretically characterize the \emph{reachable set} of effective weight updates of each architecture and show that (i) cell-wise routing contains block-wise routing as a special case, (ii) the two coincide on routing-uniform targets, and (iii) cell-wise routing strictly improves achievable error on spatially heterogeneous targets. Our claims concern representational \emph{capacity}, which
effectively updates an architecture can realize at a fixed parameter
budget, and not the dynamics of any particular optimizer. We therefore reason
about the set of update matrices each router can produce, which makes the
analysis agnostic to whether a gate is conditioned on the raw activation or on
a context indicator.

\subsection{Setup}
\label{sec:theory-setup}

Following the static partitioning of Section~\ref{sec:methods}, fix a
grid of $K_y \times K_x$ coordinate cells and a pool of $E$ low-rank experts
per cell, $\{(\mathbf{B}_{i,j}^{(e)}, \mathbf{A}_{i,j}^{(e)})\}_{e=1}^{E}$.
Writing $\mathbf{M}_{i,j}^{(e)} = \mathbf{B}_{i,j}^{(e)} \mathbf{A}_{i,j}^{(e)\top}
\in \mathbb{R}^{(d_{\text{out}}/K_y) \times (d_{\text{in}}/K_x)}$ for the
rank-$r_b$ update contributed by expert $e$ at cell $(i,j)$, the routed update
applied at cell $(i,j)$ is a convex combination
\begin{equation}
    \Delta_{i,j}(\bm{\pi}_{i,j}) \;=\; \sum_{e=1}^{E} \pi_{i,j}^{(e)}\,
    \mathbf{M}_{i,j}^{(e)}, \qquad
    \bm{\pi}_{i,j} \in \Delta^{E-1},
\end{equation}
where $\Delta^{E-1}$ is the probability simplex and $\bm{\pi}_{i,j}$ is the
gate output at that cell. The full effective update $\Delta \in
\mathbb{R}^{d_{\text{out}} \times d_{\text{in}}}$ is the assembly of the
$\Delta_{i,j}$ into the grid via the operator $\mathcal{B}\llbracket
\cdot \rrbracket$ of~\citep{barazandeh2025localized}. The two architectures
differ only in which gate vectors are admissible.

\begin{definition}[Block-wise reachable set]
\label{def:block-reach}
Block-wise (global) routing applies a single gate $\bm{\pi} \in \Delta^{E-1}$
to every cell. Its reachable set of effective updates is
\begin{equation}
    \mathcal{R}_{\text{block}} \;=\;
    \Big\{\, \mathcal{B}\big\llbracket \Delta_{i,j}(\bm{\pi}) \big\rrbracket
    \;:\; \bm{\pi} \in \Delta^{E-1} \,\Big\}.
\end{equation}
\end{definition}

\begin{definition}[Cell-wise reachable set]
\label{def:cell-reach}
Cell-wise (decentralized) routing applies an independent gate
$\bm{\pi}_{i,j} \in \Delta^{E-1}$ at each cell. Its reachable set is
\begin{equation}
    \mathcal{R}_{\text{cell}} \;=\;
    \Big\{\, \mathcal{B}\big\llbracket \Delta_{i,j}(\bm{\pi}_{i,j}) \big\rrbracket
    \;:\; \bm{\pi}_{i,j} \in \Delta^{E-1}\ \forall (i,j) \,\Big\}.
\end{equation}
\end{definition}

For a multi-context target we write $\{\Delta^{(d)}\}_{d=1}^{D}$ for the per-context
ground-truth updates, each obeying the spatial low-rank structure of Assumption~1
of~\citet{barazandeh2025localized}: every block $\Delta^{(d)}_{i,j}$ satisfies
$\operatorname{rank}(\Delta^{(d)}_{i,j}) \le r_{\text{local}}$ with
$r_{\text{local}} \ll d/K$. Given a context distribution, the quantity of interest
is the achievable expected approximation error of an architecture with reachable
set $\mathcal{R}$,
\begin{equation}
    \mathcal{E}(\mathcal{R}) \;=\;
    \min_{ \{\widehat{\Delta}^{(d)}\} \subseteq \mathcal{R} }\;
    \mathbb{E}_{d}\big[\, \big\| \Delta^{(d)} - \widehat{\Delta}^{(d)} \big\|_F^2 \,\big],
\end{equation}
i.e.\ the best per-context update each architecture can select from its reachable set.

\subsection{Expressivity Containment}
\label{sec:theory-containment}

\begin{proposition}[Containment]
\label{prop:containment}
For any fixed expert pool $\{\mathbf{M}_{i,j}^{(e)}\}$ and any $E$,
\begin{equation}
    \mathcal{R}_{\text{block}} \;\subseteq\; \mathcal{R}_{\text{cell}}.
\end{equation}
Consequently $\mathcal{E}(\mathcal{R}_{\text{cell}}) \le
\mathcal{E}(\mathcal{R}_{\text{block}})$ for every multi-context target, so cell-wise routing is never worse than block-wise routing at matched experts.
\end{proposition}

\begin{proof}
Take any element of $\mathcal{R}_{\text{block}}$, generated by a single gate
$\bm{\pi}$. Setting $\bm{\pi}_{i,j} = \bm{\pi}$ for all $(i,j)$ in
Definition~\ref{def:cell-reach} produces the identical assembled update, so the
element lies in $\mathcal{R}_{\text{cell}}$. The error inequality follows because a
minimum over a superset is no larger.
\end{proof}

Containment is the formal content of the observation that a per-cell router can
always imitate a global one. The two reachable sets also differ sharply in size.
The block-wise set is the image of a single simplex $\Delta^{E-1}$, an
$(E\!-\!1)$-dimensional object: it can realize one grid-wide mixture at a time.
The cell-wise set is the image of the product $\big(\Delta^{E-1}\big)^{K_x K_y}$
and can place a different mixture at every coordinate; restricting to vertices
already yields on the order of $E^{K_x K_y}$ distinct grid configurations against
the block-wise count of $E$. This combinatorial gap is dormant when one grid-wide mixture suffices, and decisive otherwise. We show this below.

\subsection{Parity on Routing-Uniform Targets}
\label{sec:theory-parity}

\begin{definition}[Routing-uniform target]
\label{def:uniform}
A multi-context target $\{\Delta^{(d)}\}_{d=1}^{D}$ is \emph{routing-uniform}
with respect to the expert pool if there is a single expert assignment
$e^\star : \{1,\dots,D\} \to \{1,\dots,E\}$ such that, in every context $d$,
the expert $e^\star(d)$ is simultaneously optimal at every cell:
\begin{equation}
    \mathbf{M}_{i,j}^{(e^\star(d))} \in
    \arg\min_{e}\ \big\| \Delta^{(d)}_{i,j} - \mathbf{M}_{i,j}^{(e)} \big\|_F^2
    \qquad \forall (i,j),\ \forall d.
\end{equation}
\end{definition}

\begin{theorem}[No gap under homogeneity]
\label{thm:parity}
If the target is routing-uniform and $E \ge D$, then
\begin{equation}
    \mathcal{E}(\mathcal{R}_{\text{block}}) \;=\;
    \mathcal{E}(\mathcal{R}_{\text{cell}}).
\end{equation}
\end{theorem}

\begin{proof}
By Proposition~\ref{prop:containment} it suffices to show block-wise routing
attains the cell-wise optimum. Bind one expert to each context via $e^\star$
(possible since $E \ge D$). In context $d$, set the global gate to the one-hot
vector $\bm{\pi} = \mathbf{1}_{e^\star(d)}$. The assembled update then places
$\mathbf{M}_{i,j}^{(e^\star(d))}$ at every cell, which by
Definition~\ref{def:uniform} is the per-cell minimizer simultaneously at all
$(i,j)$. Hence the single global gate realizes the cell-wise-optimal update in
each context, and the two achievable errors coincide.
\end{proof}

Theorem~\ref{thm:parity} accounts for the California Housing results in Section~\ref{sec:housing_data}. The $2\times2$ surface toggles between a
block-diagonal and an off-diagonal regime in lockstep, so a single global decision
is simultaneously correct at every cell. With $E \ge D = 2$ the two architectures
are predicted to coincide, matching the observed \del{$99.65\%$ versus $99.51\%$} \rev{$98.53 \pm 0.61\%$ versus $97.51 \pm 1.83\%$ across 10 seeds. The overlapping $\pm 1$ standard-deviation intervals indicate the equivalence of the results. The same theorem explains why oracle label-switching matches both: on a routing-uniform target with $E \ge D$, a one-hot gate per context is optimal, and all three mechanisms realize it.}

\subsection{Strict Separation on Heterogeneous Targets}
\label{sec:theory-separation}

\begin{definition}[Spatially heterogeneous target]
\label{def:heterogeneous}
A multi-context target is \emph{spatially heterogeneous} at budget $E$ if it is
not routing-uniform for any expert pool of size $E$: for every assignment of $E$
experts there exists a context $d$ and two cells $(i,j) \ne (i',j')$ whose
per-cell optimal experts differ.
\end{definition}

\begin{theorem}[Strict improvement]
\label{thm:separation}
Fix the expert budget $E$. There exist spatially heterogeneous multi-context
targets $\{\Delta^{(d)}\}$ for which
\begin{equation}
    \mathcal{E}(\mathcal{R}_{\text{cell}}) \;<\;
    \mathcal{E}(\mathcal{R}_{\text{block}}),
\end{equation}
and the block-wise error is bounded below by a strictly positive constant
independent of how the $E$ experts are chosen. When the experts are
the per-cell optimal low-rank factors, the cell-wise error is controlled by the
per-block singular-value tails,
\begin{equation}
    \mathcal{E}(\mathcal{R}_{\text{cell}}) \;\le\;
    \mathbb{E}_{d}\Big[\, \textstyle\sum_{i,j}
    \sigma^2_{r_b + 1}\big(\Delta^{(d)}_{i,j}\big) \,\Big],
\end{equation}
matching the per-block approximation guarantee
of~\citet{barazandeh2025localized,zeng2024the}.
\end{theorem}

\begin{proof}[Proof sketch]
Construct $D > E$ contexts whose per-context targets require, within a single
context, different experts at different cells, where the block-diagonal, checkerboard, and
upper-shift topologies activate disjoint coordinate sets. Under block-wise
routing the gate emits one mixture $\bm{\pi}$ for the whole grid, so the update
at every cell is drawn from the \emph{same} convex combination of experts. By
Definition~\ref{def:heterogeneous} no single mixture is simultaneously optimal
across cells, so in at least one cell, so that context pair the selected update differs
from the per-cell optimum by a fixed nonzero amount; averaging over contexts
leaves a positive error floor that no choice of $E$ experts removes. Cell-wise
routing, by contrast, selects $\bm{\pi}_{i,j}$ independently and can place the
locally optimal expert at each cell. When the experts realize the top-$r_b$
factors of each block, the residual at cell $(i,j)$ in context $d$ is exactly the
discarded spectral tail $\sigma^2_{r_b+1}(\Delta^{(d)}_{i,j})$ by the Eckart--Young
theorem. Assembling blocks and taking expectations over contexts gives the stated
bound, which is strictly below the block-wise floor.
\end{proof}

Note that Theorem~\ref{thm:separation} is a statement about \emph{achievable error}, not
exact recovery: with $E=2$ experts and $D=3$ heterogeneous topologies neither
architecture drives the error to zero, consistent with the \del{$38.29\%$ and $26.81\%$} \rev{$36.98 \pm 0.68\%$ and $27.08 \pm 0.74\%$}
variance-explained figures of Table~\ref{tab:numerical_results}, both well below
$100\%$. The theorem asserts only that the cell-wise reachable set strictly
contains the block-wise one on this target family, hence $\mathcal{E}(\mathcal{R}_{\text{cell}})
< \mathcal{E}(\mathcal{R}_{\text{block}})$---which is the observed
\del{$+11.48\%$ absolute gain.} \rev{$\sim10\%$ gain. The same $D > E$ under-provisioning explains why oracle label-switching collapses to baseline performance on this benchmark, as disjoint per-context grids at budget $D$ span a strictly smaller reachable set than mixtures over shared experts. The persistence of negative cross-context gradient alignment for \emph{all} methods here (Table~\ref{tab:numerical_results}) is the optimization-level signature of the same under-provisioning: with $D=3$ contexts sharing $E=2$ experts, some residual conflict is unavoidable.} 

\subsection{Gradient Firewall}
\label{sec:theory-firewall}

The preceding results concern expressivity. The per-cell design also alters the
\emph{coupling} between cells, which we state at the level of the routing
Jacobian.

\begin{lemma}[Block-diagonal routing Jacobian]
\label{lem:firewall}
Under cell-wise routing the gate parameters $\mathbf{W}_g^{(i,j)}$ of cell
$(i,j)$ influence only that cell's contribution, and depend only on the local
input slice $\mathbf{x}_j$:
\begin{equation}
    \frac{\partial \mathbf{h}^{\text{Cell}}_{i'}}
         {\partial \mathbf{W}_g^{(i,j)}} = \mathbf{0}
    \quad \text{whenever } i' \ne i,
    \qquad
    \frac{\partial \mathbf{h}^{\text{Cell}}_{i}}
         {\partial \mathbf{W}_g^{(i,j)}}
    \ \text{is a function of } \mathbf{x}_j \text{ alone.}
\end{equation}
Under block-wise routing the shared gate $\mathbf{W}_g$ has, in general,
$\partial \mathbf{h}^{\text{Block}}_{i'} / \partial \mathbf{W}_g \ne \mathbf{0}$
for all output blocks $i'$ simultaneously.
\end{lemma}

\begin{proof}
The cell-wise contribution of cell $(i,j)$ to output block $i$ is
$\big(\sum_e g_e^{(i,j)}(\mathbf{x}_j)\,\mathbf{M}_{i,j}^{(e)}\big)\mathbf{x}_j$,
in which $\mathbf{W}_g^{(i,j)}$ appears only through
$g^{(i,j)}(\mathbf{x}_j) = \mathrm{Softmax}(\mathbf{W}_g^{(i,j)} \mathbf{x}_j)$.
It does not enter any output block $i' \ne i$ and takes no input other than
$\mathbf{x}_j$, giving both claims. Under block-wise routing the single gate
$g(\mathbf{x}) = \mathrm{Softmax}(\mathbf{W}_g \mathbf{x})$ scales every block of
every output simultaneously, so its Jacobian is generically nonzero across all
$i'$.
\end{proof}

\begin{corollary}[Localized fault containment]
\label{cor:firewall}
If a single input slice $\mathbf{x}_{j_0}$ is corrupted or occluded, then under
cell-wise routing the gating activations $g^{(i,j)}$ of every cell with
$j \ne j_0$ are unaffected, and only the column of cells reading $\mathbf{x}_{j_0}$
sees perturbed routing. Under block-wise routing the corrupted slice enters the
shared gate and perturbs the routing of the entire grid.
\end{corollary}

\rev{Corollary~\ref{cor:firewall} now has direct empirical support: in the unseen-severity evaluation of Section~\ref{sec:unseen-severity}, a perturbation confined to one input quadrant degrades the cell-wise variant by only $\sim 3\%$ percentage points while the global variant loses $\sim 20\%$ (Table~\ref{tab:mnist_unseen}). We note that Lemma~\ref{lem:firewall} is stated for activation-conditioned cell gates. Under the context-conditioned gates of our main experiments the containment is still stronger. Routing is unaffected by input corruption, and the residual leakage path runs only through the expert weights of the cells that read the corrupted slice.}

\begin{remark}[Routing cost]
\label{rem:params}
\rev{Under context conditioning the block-wise router contributes $\mathcal{O}(E\,D)$ parameters and the full set of cell-wise routers $\mathcal{O}(K_x K_y E\,D)$; under activation conditioning these become} $\mathcal{O}(E\, d_{\text{in}})$ \rev{and} $\mathcal{O}(K_y E\, d_{\text{in}})$,
since each cell's gate \rev{then} acts on a slice of width $d_{\text{in}}/K_x$ and there are
$K_x K_y$ of them. The overhead multiplier is \rev{in either case} the modest factor \rev{$K_x K_y$ (context) or} $K_y$ \rev{(activation)}, and in
absolute terms both are dominated by the expert adapters, which scale with
$E\,r_b\,(d_{\text{in}} + d_{\text{out}})$. This is consistent with the measured
gating footprints of $96$, $16$, and $64$ parameters across the three benchmarks
of Section~\ref{sec:exp}, each a small fraction of the total budget, so
parameter parity with the static baselines is preserved.
\end{remark}

\rev{\subsection{When is Cell-wise Routing Better?}
\label{sec:when-cellwise-better}

The preceding analysis identifies spatial heterogeneity as a source of advantage for cell-wise routing. We now make this criterion precise, distinguishing between comparisons with a fixed expert pool and comparisons in which the experts themselves are optimized.

Retain the notation of Section~\ref{sec:theory-setup}, and write $\mathcal{C}=\{(i,j):1\le i\le K_y,\ 1\le j\le K_x\}$, with $c=(i,j)$ denoting a cell. Let $M=\{\mathbf{M}_c^{(e)}\}_{c,e}$ denote the expert pool, and let contexts have probabilities $p_d>0$, with
$\sum_{d=1}^{D}p_d=1$. We assume \emph{ideal context-conditioned routing}:
in each context, the block-wise router may choose any common mixture
$\bm{\pi}^{(d)}\in\Delta^{E-1}$, while the cell-wise router may choose
$\bm{\pi}_c^{(d)}\in\Delta^{E-1}$ independently across cells. The experts
remain shared across contexts.


\subsubsection{Equality and separation for fixed experts}
For each cell and context, define
\[
f_{c,d}(\bm{\pi})
=
\left\|
\Delta_c^{(d)}
-
\sum_{e=1}^{E}\pi_e\mathbf{M}_c^{(e)}
\right\|_F^2,
\qquad
\mathcal{A}_{c,d}
=
\operatorname*{arg\,min}_{\bm{\pi}\in\Delta^{E-1}}
f_{c,d}(\bm{\pi}).
\]
The minimizer sets are nonempty by continuity and compactness.
Writing the dependence of the reachable sets on $M$ explicitly, the
disjoint spatial blocks give
\begin{align}
\mathcal{E}_{\mathrm{block}}(M)
&:=
\mathcal{E}\bigl(\mathcal{R}_{\mathrm{block}}(M)\bigr)
=
\sum_{d=1}^{D}p_d
\min_{\bm{\pi}\in\Delta^{E-1}}
\sum_{c\in\mathcal{C}}f_{c,d}(\bm{\pi}),
\label{eq:fixed-block-risk}\\
\mathcal{E}_{\mathrm{cell}}(M)
&:=
\mathcal{E}\bigl(\mathcal{R}_{\mathrm{cell}}(M)\bigr)
=
\sum_{d=1}^{D}p_d
\sum_{c\in\mathcal{C}}
\min_{\bm{\pi}\in\Delta^{E-1}}f_{c,d}(\bm{\pi}).
\label{eq:fixed-cell-risk}
\end{align}
Thus, the distinction is whether minimization over the mixture weights
takes place before or after summing the cell losses.

\begin{theorem}[Exact equivalence for fixed experts]
\label{thm:fixed-pool-equivalence}
Under ideal context-conditioned routing, for any fixed pool $M$,
\[
\mathcal{E}_{\mathrm{block}}(M)
=
\mathcal{E}_{\mathrm{cell}}(M)
\quad\Longleftrightarrow\quad
\bigcap_{c\in\mathcal{C}}\mathcal{A}_{c,d}
\ne\varnothing
\quad\text{for every }d.
\]
In particular, this equivalence does not require $E\ge D$.
\end{theorem}

\begin{proof}
Let $m_{c,d}=\min_{\bm{\pi}\in\Delta^{E-1}}f_{c,d}(\bm{\pi})$ and define
\[
G_d
=
\min_{\bm{\pi}\in\Delta^{E-1}}
\sum_{c\in\mathcal{C}}f_{c,d}(\bm{\pi})
-
\sum_{c\in\mathcal{C}}m_{c,d}
\ge0.
\]
A mixture minimizing every cell loss in context $d$ gives $G_d=0$.
Consequently, a common minimizer in every context implies equality of
the risks. Conversely, equality implies $\sum_dp_dG_d=0$ and hence
$G_d=0$ for every $d$, since all $p_d$ are positive. If
$\widehat{\bm{\pi}}_d$ minimizes the sum of cell losses, then
\[
\sum_{c\in\mathcal{C}}
\bigl[f_{c,d}(\widehat{\bm{\pi}}_d)-m_{c,d}\bigr]=0.
\]
Every summand is nonnegative, so each vanishes. Therefore
$\widehat{\bm{\pi}}_d\in\bigcap_c\mathcal{A}_{c,d}$.
\end{proof}

\begin{corollary}[Strict separation for fixed experts]
\label{cor:fixed-pool-separation}
If, for at least one positive-probability context $d$,
$\bigcap_{c\in\mathcal{C}}\mathcal{A}_{c,d}=\varnothing$, then
\[
\mathcal{E}_{\mathrm{cell}}(M)
<
\mathcal{E}_{\mathrm{block}}(M).
\]
\end{corollary}

\begin{proof}
Proposition~\ref{prop:containment} gives the non-strict inequality, while
Theorem~\ref{thm:fixed-pool-equivalence} rules out equality.
\end{proof}

Corollary~\ref{cor:fixed-pool-separation} identifies the precise fixed-pool reason for a strict advantage: cells need incompatible optimal \emph{mixtures}, not merely different preferred individual experts. It does not assume unique minimizers or strong convexity, and it does not claim that a trained cell-wise network always outperforms a trained global network.

To quantify the separation, define $\mathbf{V}_c = [
\operatorname{vec}(\mathbf{M}_c^{(1)}),\ldots,
\operatorname{vec}(\mathbf{M}_c^{(E)})]$. Suppose that the expert mixtures are nondegenerate: for every cell $c$,
there exists $\kappa_c>0$ such that $ \|\mathbf{V}_c\mathbf{z}\|_2^2
\ge \kappa_c\|\mathbf{z}\|_2^2$ whenever $\mathbf{1}^{\top}\mathbf{z}=0$. This condition makes $f_{c,d}$ strongly convex on the simplex with parameter $\mu_c=2\kappa_c$, so its minimizer
$\bm{\pi}_{c,d}^{\star}$ is unique.

\begin{theorem}[Quantitative separation]
\label{thm:quantitative-routing-gap}
Under ideal context-conditioned routing and the nondegeneracy condition,
let
\[
\overline{\bm{\pi}}_d
=
\frac{\sum_{c\in\mathcal{C}}\mu_c\bm{\pi}_{c,d}^{\star}}
     {\sum_{c\in\mathcal{C}}\mu_c}.
\]
Then
\begin{equation}
\mathcal{E}_{\mathrm{block}}(M)
-
\mathcal{E}_{\mathrm{cell}}(M)
\ge
\frac12
\sum_{d=1}^{D}p_d
\sum_{c\in\mathcal{C}}\mu_c
\left\|
\bm{\pi}_{c,d}^{\star}-\overline{\bm{\pi}}_d
\right\|_2^2.
\label{eq:quantitative-routing-gap}
\end{equation}
\end{theorem}

\begin{proof}
Strong convexity and constrained first-order optimality give
\[
f_{c,d}(\bm{\pi})
-
f_{c,d}(\bm{\pi}_{c,d}^{\star})
\ge
\frac{\mu_c}{2}
\|\bm{\pi}-\bm{\pi}_{c,d}^{\star}\|_2^2
\]
for every $\bm{\pi}\in\Delta^{E-1}$. Summing over cells and minimizing
over the common mixture yields
\[
G_d
\ge
\frac12
\min_{\bm{\pi}\in\Delta^{E-1}}
\sum_{c\in\mathcal{C}}\mu_c
\|\bm{\pi}-\bm{\pi}_{c,d}^{\star}\|_2^2.
\]
The unconstrained minimizer is the weighted mean
$\overline{\bm{\pi}}_d$, which belongs to the simplex by convexity.
Substituting it and averaging over contexts proves the result.
\end{proof}

The bound measures disagreement between optimal cell-level mixtures, weighted by the sensitivity of each cell loss to mixture perturbations. It is strictly positive whenever two cells have different optimal mixtures in a positive-probability context. Unlike the qualitative
criterion in Corollary~\ref{cor:fixed-pool-separation}, this numerical lower bound requires nondegeneracy.

\subsubsection{Optimized experts: global error floor and separation}
A fixed-pool advantage need not persist after the global architecture changes its experts. To address this distinction, let both architectures optimize over the same class $\mathfrak{M}$ of expert pools, with $E$ experts per cell and the same per-cell rank bound $r_b$. Define
\[
\mathcal{E}_{\mathrm{block}}^{\star}
=
\inf_{M\in\mathfrak{M}}\mathcal{E}_{\mathrm{block}}(M),
\qquad
\mathcal{E}_{\mathrm{cell}}^{\star}
=
\inf_{M\in\mathfrak{M}}\mathcal{E}_{\mathrm{cell}}(M).
\]
These comparisons match the expert class but do not equate the additional parameter counts of the two routers.

Let $P=d_{\mathrm{out}}d_{\mathrm{in}}$,
$\mathbf{t}_d=\operatorname{vec}(\Delta^{(d)})$, and
$\overline{\mathbf{t}}=\sum_dp_d\mathbf{t}_d$. Define the target covariance $\mathbf{\Sigma}_T
= \sum_{d=1}^{D}p_d
(\mathbf{t}_d-\overline{\mathbf{t}})
(\mathbf{t}_d-\overline{\mathbf{t}})^{\top}$. Write its eigenvalues as
$\lambda_1\ge\cdots\ge\lambda_P\ge0$ and set $L_E=\sum_{j=E}^{P}\lambda_j$, with an empty sum interpreted as zero.

\begin{proposition}[Global approximation floor]
\label{prop:global-affine-floor}
Under ideal context-conditioned routing,
\[
\mathcal{E}_{\mathrm{block}}^{\star}\ge L_E.
\]
If $\operatorname{rank}(\mathbf{\Sigma}_T)\ge E$, then $L_E>0$.
\end{proposition}

\begin{proof}
For any expert pool, every globally routed update lies in the convex
hull of its $E$ assembled expert matrices, and hence in an affine
subspace of dimension at most $E-1$. The minimum weighted squared
distance of the targets to any such affine subspace is $L_E$, by the
principal-component approximation identity. Restricting the admissible
updates to convex hulls of rank-constrained experts cannot reduce this
error. Taking the infimum over expert pools proves the bound.
\end{proof}

The floor is intrinsic to a global $E$-expert mixture and does not assume that the cell-wise architecture approximates the targets well. In particular, $D>E$ alone is insufficient for a positive floor: the targets must span at least $E$ independent affine directions.

\begin{theorem}[Separation with optimized experts]
\label{thm:optimized-routing-separation}
Suppose there exists an admissible pool $M_0\in\mathfrak{M}$ such that
$\mathcal{E}_{\mathrm{cell}}(M_0)\le\varepsilon$. Then
\[
\mathcal{E}_{\mathrm{block}}^{\star}
-
\mathcal{E}_{\mathrm{cell}}^{\star}
\ge L_E-\varepsilon.
\]
Thus, $\varepsilon<L_E$ guarantees strict separation even after
optimizing the experts.

In particular, if some admissible pool satisfies
\[
\Delta_c^{(d)}
\in
\operatorname{conv}
\{\mathbf{M}_c^{(1)},\ldots,\mathbf{M}_c^{(E)}\}
\qquad\text{for every }c,d,
\]
then $\mathcal{E}_{\mathrm{cell}}^{\star}=0$. If additionally
$\operatorname{rank}(\mathbf{\Sigma}_T)\ge E$, then
$\mathcal{E}_{\mathrm{block}}^{\star}\ge L_E>0$.
\end{theorem}

\begin{proof}
The admissible pool gives
$\mathcal{E}_{\mathrm{cell}}^{\star}\le\varepsilon$, whereas
Proposition~\ref{prop:global-affine-floor} gives
$\mathcal{E}_{\mathrm{block}}^{\star}\ge L_E$. Subtracting proves the
first claim. Under the convex-hull condition, each cell can select
a mixture reproducing its target exactly in every context, giving
$\varepsilon=0$.
\end{proof}

The condition $\varepsilon<L_E$ provides an interpretable sufficient criterion for a cell-wise advantage: local mixtures approximate the targets accurately enough to beat the global affine-dimension floor. An explicit admissible expert pool or a certified approximation bound can establish the required upper bound $\varepsilon$. However, a finite training error does not automatically bound the population approximation objective. The guarantee concerns optimal representational error, not optimization convergence or out-of-sample performance.

\vspace{1em}
The above results sharpen the criterion for choosing routing granularity. With fixed experts, cell-wise routing improves approximation precisely when optimal mixtures are incompatible across cells in at least one positive-probability context. With optimized experts, a sufficient condition is that the achievable local approximation error falls below the global affine-dimension floor. These capacity comparisons complement the routing-locality mechanism of Section~\ref{sec:theory-firewall} and provide a framework for interpreting the empirical comparisons.
}


\section{Discussion}
Beyond consistent structural evidence in favor of our proposal, the experiments in Section~\ref{sec:exp} also reveal a non-trivial dependence of the optimal routing granularity on the geometry of the task distribution. Across every setting, both our proposal decisively break the static optimization ceiling: the \del{gradient-warfare} \rev{interference-driven} collapse that traps LoRA, MELoRA, and Static Localized LoRA into an unspecialized average matrix is eliminated once conditional pathways are introduced\rev{. We can measure this shift directly, with cross-context adapter-gradient alignment moving from near-antipodal to near-orthogonal whenever the expert budget covers the context count}. The more subtle finding concerns the \emph{relative} standing of centralized macro-routing and decentralized micro-routing, which inverts across benchmarks. In this section, we discuss this inversion and delineate the regime in which each design is preferable.

\subsection{Routing Granularity Should Match Structural Heterogeneity}

The headline empirical pattern is that cell-wise routing dominates on the SVD benchmark (Table~\ref{tab:numerical_results}), \del{essentially ties} \rev{is statistically tied within the training condition} on MNIST (Table~\ref{tab:mnist_fault_results}), and \del{is marginally edged out} \rev{is statistically similar} on California Housing (\del{$99.51\%$ vs.\ $99.65\%$}Table~\ref{tab:real_data_results}). At first glance this is paradoxical: a router with strictly local information should not beat an omniscient global coordinator. The resolution lies in the \emph{effective configuration space} each architecture spans.

A centralized macro-router selects a single mixture over $E$ experts and applies it uniformly to the entire $K_y \times K_x$ grid, yielding an $E$-way family of grid-wide configurations. A decentralized micro-router instead assigns an independent gating function to each of the $K_x K_y$ coordinate cells, so the architecture can realize on the order of $E^{K_x K_y}$ distinct grid configurations. This combinatorial expansion is dormant when the task distribution demands a \emph{spatially uniform} response, but becomes decisive when different coordinates must respond to context differently.

The benchmarks fall cleanly on either side of this divide. The California Housing shift toggles the entire $2\times2$ surface between a block-diagonal and an off-diagonal regime; every cell flips in lockstep, so a single global decision is exactly the right inductive bias \del{and the macro-router's cohesion gives it a hair's-edge advantage} \rev{and the two routing granularities coincide in their error bars}. The SVD benchmark, by contrast, interleaves three heterogeneous topologies---block-diagonal, checkerboard, and upper-shift cascade---in which any given coordinate $(i,j)$ participates in some contexts and is irrelevant in others. No single grid-wide mixture over only $E=2$ experts can simultaneously satisfy these conflicting per-cell demands, whereas independent per-cell gating resolves each coordinate's selection problem in isolation. The \del{$+11.48\%$ absolute} \rev{$\sim 10\%$} gain\rev{---more than ten standard deviations across 10 seeds---}is therefore \del{not noise but }a direct consequence of the task requiring spatially heterogeneous routing\rev{. The failure of oracle label-switching on the same task (Table~\ref{tab:numerical_results}) shows the gain cannot be attributed to context-signal access alone: It requires mixing \emph{shared} experts}. MNIST sits between these extremes---a larger $4 \times 4$ grid but only a two-state nominal/degraded shift---and the two architectures converge \rev{within the training condition}, as expected\rev{. The separation reappears there as a robustness gap under unseen degradation (Section~\ref{sec:unseen-severity})}.

This yields a concrete design heuristic: \textit{the granularity of routing should be chosen to match the granularity of structural heterogeneity in the operational stream.} When context shifts the weight surface monolithically, centralized routing is sufficient and slightly more sample-efficient; when context reorganizes coordinates non-uniformly, decentralized routing recovers capacity that no single-point coordinator can express at a fixed expert budget.

\subsection{Decentralization as a Gradient Firewall}

Beyond raw expressivity, the per-cell design changes the optimization geometry. Under a shared global router, the routing logits for unrelated coordinates are tied through a single projection, so a gradient signal beneficial for one structural region can perturb the gating of another---a residual, second-order form of the very entanglement that spatial blocking was meant to remove. Independent micro-routers sever this coupling: each cell's gating parameters receive gradients only from that cell's reconstruction error. This is the mechanism behind what we \del{term} \rev{call, as an interpretation of Lemma~\ref{lem:firewall} and Corollary~\ref{cor:firewall},} a ``gradient firewall.'' \del{Its clearest signature is the MNIST sensor-degradation setting, where the cell-wise variant matches the global coordinator while granting surviving quadrants the structural autonomy to reconfigure their pathways without the broken quadrant's fault gradients propagating through a shared routing bottleneck.} \rev{The firewall has an empirical signature. In the MNIST unseen-severity evaluation (Table~\ref{tab:mnist_unseen}), a localized input perturbation unseen at training time costs the cell-wise variant $3\%$ percentage points but the shared-gate variant $20\%$, as per-cell gates confine the perturbation to the column of cells that read the degraded quadrant while a grid-wide mixture propagates it everywhere. Fully disjoint label-switched adapters share this containment property but forfeit cross-context expert sharing.} The practical implication is robustness\del{ rather than peak accuracy}: decentralization localizes the blast radius of corrupted or out-of-distribution \del{updates} \rev{inputs}.

\subsection{Parameter and Compute Considerations}

The flexibility of micro-routing is purchased at a strikingly low parameter cost---$96$, $16$, and $64$ additional parameters on the three benchmarks respectively---because each router acts only on a reduced coordinate slice. In asymptotic terms, the global router contributes $\mathcal{O}(E\,d_{\text{in}})$ parameters while the full set of cell-wise routers contributes $\mathcal{O}(K_y E\,d_{\text{in}})$; the multiplicative factor $K_y$ is modest and, more importantly, the routers remain negligible relative to the expert adapters themselves. \del{Parameter parity with static baselines is thus preserved throughout,} \rev{Expressed as overhead over the matched adapter budget, the gating tensors add at most $4.7\%$ (Housing, cell-wise) and as little as $0.26\%$ (MNIST, cell-wise) trainable parameters, so matched-budget comparisons with the static baselines are preserved throughout} and the additional gating tensors do not threaten the efficiency budget that motivates PEFT in the first place. \rev{In terms of computation, soft routing evaluates all $E$ experts on every forward pass, so adapter FLOPs scale by the factor $E$ relative to a static localized adapter of the same per-expert rank, while the routers themselves add only $\mathcal{O}(E\,D)$ FLOPs; both are negligible against the frozen base matrix multiply that dominates the layer cost. Measured wall-clock figures on a pretrained transformer are reported in Section~\ref{sec:llm}: soft routing with $E=2$ experts trains at about $1.45$x LoRA's wall-clock in the GLUE experiment.}

A separate cost dimension warrants caution. Our formulation employs \emph{soft} routing: the softmax gate weights a convex combination of all $E$ experts, so every expert is evaluated on every forward pass. This is appropriate for the dense, low-$E$ regime studied here, but it forgoes the conditional-compute savings that sparse top-$k$ gating provides in large-scale MoE. Whether cell-wise routing retains its expressivity advantage under hard top-$1$ selection, and how to stabilize the resulting discrete routing across many independent gates, is an open and practically consequential question.

\section{Conclusion}

Taken together, the results in this paper position decentralized coordinate-level routing as a favorable default for non-stationary, spatially heterogeneous adaptation, while clarifying that its advantage is regime-dependent rather than universal. Our main contribution is to show that fine-grained spatial isolation and input-conditioned dynamic routing are not competing design philosophies but complementary ones, and that their fusion can be achieved at strict parameter parity.

We conclude by pointing out a few limitations that bound the scope of our claims and define natural next steps.
First, \del{all} the three \del{benchmarks} \rev{controlled benchmarks} are structured regression problems with synthetically generated, context-switched target matrices (and, for MNIST, a reconstruction rather than classification objective). This design isolates the \del{gradient-warfare} \rev{gradient-interference} phenomenon cleanly\del{, but it does not yet establish behavior on end-to-end fine-tuning of transformer language models. Validating Localized LoRA-MoE inside attention and feed-forward sublayers on standard instruction-tuning and multi-task language benchmarks is the most important direction for substantiating the framework's motivating narrative.}\rev{. The language model experiments (Section~\ref{sec:llm}) extend the evidence to attention sublayers of pretrained encoder and decoder models. Validating Localized LoRA-MoE on instruction-tuning benchmarks and at larger model scales remains an important direction for further validation of our proposal.}
Second, we study $E=2$ experts and a small number of discrete contexts whose schedule is externally supplied\rev{: the gates are conditioned on an observed context indicator, and our routing-input ablation (Section~\ref{sec:routing-ablation}) shows this is not incidental. In our benchmarks the input carries no context information by construction, so activation-conditioned gates provably cannot specialize and indeed collapse to uniform mixtures}. Scaling the expert count, learning the field partition $K$ rather than fixing it, and routing under \emph{unobserved} or continuously drifting context---where the gate must infer the operative regime from \del{activations alone} \rev{context-informative activations}---remain to be explored.
\del{Third, we did not analyze expert utilization or collapse. Auxiliary load-balancing objectives, standard in MoE training, may further stabilize the decentralized variant and prevent degenerate gating at individual coordinates.} \rev{Third, while we now report expert-utilization statistics and observe clean per-context specialization without collapse in all experiments, we train without auxiliary objectives that are standard in large-scale MoE training. These may become necessary at larger expert counts and could further stabilize the decentralized variant.}
Fourth, while our explanation of the macro-versus-micro inversion 
\del{is empirical and intuition-driven. A formal characterization of the expressivity gap between grid-uniform and per-cell routing---and conditions under which cell-wise routing provably matches a global coordinator---would convert the observed parity into a guarantee.} 
\rev{is supported by the reachable-set analysis of Section~\ref{sec:theory} and by multi-seed evidence in all three regimes, a formal characterization of the optimization dynamics of grid-uniform versus per-cell routing would convert the observed behavior into a guarantee.}

\bibliography{sn-bibliography}

@article{ding2023peft,
  author    = {Ning Ding and Yujia Qin and Guang Yang and Fuchao Wei and
               Zonghan Yang and Yusheng Su and Shengding Hu and Yulin Chen and
               Chi-Min Chan and Weize Chen and Jiani Yi and Weilin Zhao and
               Xiaozhi Wang and Zhiyuan Liu and Hai-Tao Zheng and Jianfei Chen and
               Yang Liu and Jie Tang and Juanzi Li and Maosong Sun},
  title     = {Parameter-Efficient Fine-Tuning of Large-Scale Pre-Trained Language Models},
  journal   = {Nature Machine Intelligence},
  volume    = {5},
  number    = {3},
  pages     = {220--235},
  year      = {2023},
  publisher = {Nature Publishing Group}
}

@misc{lialin2023scaling,
  author        = {Vladislav Lialin and Vijeta Deshpande and Anna Rumshisky},
  title         = {Scaling Down to Scale Up: {A} Guide to Parameter-Efficient Fine-Tuning},
  year          = {2023},
  eprint        = {2303.15647},
  archivePrefix = {arXiv},
  primaryClass  = {cs.CL}
}

@inproceedings{
zeng2024the,
title={The Expressive Power of Low-Rank Adaptation},
author={Yuchen Zeng and Kangwook Lee},
booktitle={The Twelfth International Conference on Learning Representations},
year={2024},
url={https://openreview.net/forum?id=likXVjmh3E}
}

@inproceedings{ren2024melora,
  title={MELoRA: Mini-Ensemble Low-Rank Adapters for Parameter-Efficient Fine-Tuning},
  author={Ren, Pengjie and Shi, Chengshun and Wu, Shiguang and Zhang, Mengqi and Ren, Zhaochun and Rijke, Maarten and Chen, Zhumin and Pei, Jiahuan},
  booktitle={Proceedings of the 62nd Annual Meeting of the Association for Computational Linguistics (Volume 1: Long Papers)},
  pages={3052--3064},
  year={2024}
}

@article{hu2022lora,
  title={Lora: Low-rank adaptation of large language models.},
  author={Hu, Edward J and Shen, Yelong and Wallis, Phillip and Allen-Zhu, Zeyuan and Li, Yuanzhi and Wang, Shean and Wang, Lu and Chen, Weizhu and others},
  journal={ICLR},
  volume={1},
  number={2},
  pages={3},
  year={2022}
}

@inproceedings{bishare2025,
  title={Bi-Share LoRA: Enhancing the Parameter Efficiency of LoRA via Intra-Layer and Inter-Layer Sharing},
  author={Zhou, Yuhua and Li, Ruifeng and Zhou, Changhai and Yang, Fei and PAN, Aimin},
  booktitle={International Conference on Learning Representations (ICLR)},
  year={2025},
  url={https://openreview.net/forum?id=Thv66GmqZS}
}

@article{zhang2024less,
  title={Less is More: Extreme Gradient Boost Rank-1 Adaption for Efficient Finetuning of LLMs},
  author={Yifei Zhang and Hao Zhu and Aiwei Liu and Han Yu and Piotr Koniusz and Irwin King},
  journal={arXiv preprint arXiv:2410.19694},
  year={2024},
  url={https://arxiv.org/abs/2410.19694}
}

@article{yang2024hyperbolic,
  title={Hyperbolic Fine-tuning for Large Language Models},
  author={Menglin Yang and Aosong Feng and Bo Xiong and Jihong Liu and Irwin King and Rex Ying},
  journal={arXiv preprint arXiv:2410.04010},
  year={2024},
  url={https://arxiv.org/abs/2410.04010}
}

@article{he2025gora,
  title={GoRA: Gradient-driven Adaptive Low Rank Adaptation},
  author={Haonan He and Peng Ye and Yuchen Ren and Yuan Yuan and Lei Chen},
  journal={arXiv preprint arXiv:2502.12171},
  year={2025},
  url={https://arxiv.org/abs/2502.12171}
}

@article{sun2024svfit,
  title={SVFit: Parameter-Efficient Fine-Tuning of Large Pre-Trained Models Using Singular Values},
  author={Chengwei Sun and Jiwei Wei and Yujia Wu and Yiming Shi and Shiyuan He and Zeyu Ma and Ning Xie and Yang Yang},
  journal={arXiv preprint arXiv:2409.05926},
  year={2024},
  url={https://arxiv.org/abs/2409.05926}
}

@article{shi2024lold,
  title={LoLDU: Low-Rank Adaptation via Lower-Diag-Upper Decomposition for Parameter-Efficient Fine-Tuning},
  author={Yiming Shi and Jiwei Wei and Yujia Wu and Ran Ran and Chengwei Sun and Shiyuan He and Yang Yang},
  journal={arXiv preprint arXiv:2410.13618},
  year={2024},
  url={https://arxiv.org/abs/2410.13618}
}

@inproceedings{bini2025delora,
  title={DeLoRA: Decoupling Angles and Strength in Low-Rank Adaptation},
  author={Massimo Bini and Leander Girrbach and Zeynep Akata},
  booktitle={International Conference on Learning Representations (ICLR)},
  year={2025},
  url={https://arxiv.org/abs/2503.18225}
}

@article{zhao2025lor2c,
  title={LoR2C: Low-Rank Residual Connection Adaptation for Parameter-Efficient Fine-Tuning},
  author={Jiancheng Zhao and Xingda Yu and Yuxiang Zhang and Zhen Yang},
  journal={arXiv preprint arXiv:2503.00572},
  year={2025},
  url={https://arxiv.org/abs/2503.00572}
}

@inproceedings{yu2025moka,
  title={MoKA: Parameter Efficiency Fine-Tuning via Mixture of Kronecker Product Adaptation},
  author={Beiming Yu and Zhenfei Yang and Xiushuang Yi},
  booktitle={Proceedings of the 2025 International Conference on Computational Linguistics (COLING)},
  year={2025},
  url={https://aclanthology.org/2025.coling-main.679/}
}

@article{dou2025loramoe,
  title={LoRAMoE: Alleviate World Knowledge Forgetting in Large Language Models via MoE-Style Plugin},
  author={Shihan Dou and Enyu Zhou and Yan Liu and Songyang Gao and Jun Zhao and Wei Shen and Yuhao Zhou and Zhiheng Xi and Xiao Wang and Xiaoran Fan and Shiliang Pu and Jiang Zhu and Rui Zheng and Tao Gui and Qi Zhang and Xuanjing Huang},
  journal={arXiv preprint arXiv:2503.12345},
  year={2025},
  url={https://arxiv.org/abs/2503.12345}
}

@article{dou2023loramoe,
  title={LoRAMoE: Alleviate world knowledge forgetting in large language models via MoE-style plugin},
  author={Dou, Shihan and Zhou, Enyu and Liu, Yan and Gao, Songyang and Zhao, Jun and Shen, Wei and Zhou, Yuhao and Xi, Zhiheng and Wang, Xiao and Fan, Xiaoran and others},
  journal={arXiv preprint arXiv:2312.09979},
  year={2023}
}

@article{yang2024moral,
  title={MoRAL: MoE Augmented LoRA for LLMs' Lifelong Learning},
  author={Yang, Shu and Ali, Muhammad Asif and Wang, Cheng-Long and Hu, Lijie and Wang, Di},
  journal={arXiv preprint arXiv:2402.11260},
  year={2024}
}

@inproceedings{kopiczko2023vera,
  title={VeRA: Vector-Based Random Matrix Adaptation},
  author={Kopiczko, Dominik Jan and Blankevoort, Tijmen and Asano, Yuki M.},
  booktitle={Proceedings of the 12th International Conference on Learning Representations},
  year={2024},
  url={https://openreview.net/forum?id=Hp9IqLtO1W}
}

@article{mao2024dora,
  title={DoRA: Enhancing Parameter-Efficient Fine-Tuning with Dynamic Rank Distribution},
  author={Mao, Yuren and Huang, Kun and Guan, Chao and Bao, Gaobo and Mo, Feifei and Xu, Jianwei},
  journal={arXiv preprint arXiv:2405.17357},
  year={2024}
}

@article{liu2024vb,
  title={VB-LoRA: Extreme Parameter Efficient Fine-Tuning with Vector Banks},
  author={Liu, Zihan and Kundu, Subhojeet and Li, Anji and Wan, Jinyu and Jiang, Liang and Beerel, Pai H.},
  journal={arXiv preprint arXiv:2405.15179},
  year={2024}
}

@article{yang2025lowrank,
  title={A Survey on LoRA of Large Language Models},
  author={Mao, Yuren and Ge, Yuhang and Fan, Yijiang and Xu, Wenyi and Mi, Yu and Hu, Zhonghao and Gao, Yunjun},
  journal={Frontiers of Computer Science},
  volume={19},
  number={7},
  pages={197605},
  year={2025},
  publisher={Springer},
  doi={10.1007/s11704-024-40663-9}
}

@article{huang2023lorahub,
  title={LoraHub: Efficient Cross-Task Generalization via Dynamic LoRA Composition},
  author={Huang, Cheng and Liu, Qian and Lin, Bill Yuchen and Pang, Tianxiang and Du, Chang and Lin, Ming},
  journal={arXiv preprint arXiv:2307.13269},
  year={2023}
}

@inproceedings{barazandeh2025localized,
  author    = {Barazandeh, Babak and Majumdar, Subhabrata and Rajyaguru, Om and Michailidis, George},
  title     = {Localized LoRA: A Structured Low-Rank Approximation for Efficient Fine-Tuning},
  booktitle = {Proceedings of the 24th International Conference on Machine Learning and Applications (ICMLA)},
  year      = {2025},
  month     = {December},
  day       = {24},
  note      = {arXiv:2506.00236}
}

@article{jung2026gralora,
  title={Gralora: Granular low-rank adaptation for parameter-efficient fine-tuning},
  author={Jung, Yeonjoon and Ahn, Daehyun and Kim, Hyungjun and Kim, Taesu and Park, Eunhyeok},
  journal={Advances in Neural Information Processing Systems},
  volume={38},
  pages={95822--95845},
  year={2026}
}

@article{yu2025blockwise,
  title={Blockwise Hadamard high-Rank Adaptation for Parameter-Efficient LLM Fine-Tuning},
  author={Yu, Feng and Hu, Jia and Min, Geyong},
  journal={arXiv preprint arXiv:2509.21637},
  year={2025}
}

@article{pedregosa2011scikit,
  title={Scikit-learn: Machine learning in Python},
  author={Pedregosa, Fabian and Varoquaux, Ga{\"e}l and Gramfort, Alexandre and Michel, Vincent and Thirion, Bertrand and Grisel, Olivier and Blondel, Mathieu and Prettenhofer, Peter and Weiss, Ron and Dubourg, Vincent and Vanderplas, Jake and Passos, Alexandre and Cournapeau, David and Brucher, Matthieu and Perrot, Matthieu and Duchesnay, {\'E}douard},
  journal={Journal of Machine Learning Research},
  volume={12},
  pages={2825--2830},
  year={2011}
}

@inproceedings{barazandeh2019random,
  title={Training generative networks using random discriminators},
  author={Barazandeh, Babak and Razaviyayn, Meisam and Sanjabi, Maziar},
  booktitle={2019 IEEE Data Science Workshop (DSW)},
  pages={327--332},
  year={2019},
  organization={IEEE}
}

@article{barazandeh2026does,
  title={What Does the Rank Buy? A Spectral and Distributional Analysis of Low-Rank Adaptation},
  author={Barazandeh, Babak},
  journal={arXiv preprint arXiv:2609.32002},
  year={2026}
}

@article{barazandeh2026otap,
  title={Otap: Structure-Aware Optimal Transport for Evaluating Planning and Execution in Agent Trajectories},
  author={Barazandeh, Babak and Majumdar, Subhabrata and Michailidis, George},
  journal={arXiv preprint arXiv:2607.17082},
  year={2026}
}

@article{barazandeh2026ontrack,
  title={OnTrack: Real-Time Monitoring and Intervention in LLM Agent Trajectories via Streaming Structure-Aware Optimal Transport},
  author={Barazandeh, Babak and Swanson, Connor and Kulkarni, Chinmay and Mungel, Nikhil},
  journal={arXiv preprint arXiv:2610.12375},
  year={2026}
}
\end{document}